\documentclass[11pt]{article}

\usepackage{fullpage,times,url,bm}

\usepackage{amsthm,amsfonts,amsmath,amssymb,epsfig,color,float,graphicx,verbatim}
\usepackage{algorithm,algorithmic}
\usepackage{bbm}
\usepackage{enumitem}

\usepackage{graphicx}
\usepackage{subcaption}

\usepackage{array}

\usepackage[dvipsnames]{xcolor}

\usepackage[colorlinks,linkcolor = RawSienna, urlcolor  = Green, citecolor = Green, anchorcolor = ForestGreen,bookmarks=False]{hyperref}

\newtheorem{theorem}{Theorem}[section]
\newtheorem{proposition}{Proposition}[section]
\newtheorem{lemma}{Lemma}[section]

\newtheorem{definition}{Definition}[section]

\newtheorem{example}{Example}

\newcommand{\secref}[1]{Section~\ref{#1}}
\newcommand{\subsecref}[1]{Section~\ref{#1}}

\renewcommand{\eqref}[1]{Eq.~(\ref{#1})}
\newcommand{\lemref}[1]{Lemma~\ref{#1}}

\newcommand{\thmref}[1]{Theorem~\ref{#1}}

\newcommand{\appref}[1]{Appendix~\ref{#1}}

\newcommand{\assref}[1]{Assumption~\ref{#1}}
\newcommand{\exampleref}[1]{Example~\ref{#1}}

\usepackage{amssymb}

\usepackage{bbm}
\newcommand{\onefunc}{\mathbbm{1}}

\newcommand{\stam}[1]{}
\newcommand{\ignore}[1]{}
\usepackage{comment}

\usepackage{mathtools}

\newtheorem{assumption}[theorem]{Assumption}

\newcommand{\be}{\mathbf{e}}
\newcommand{\bx}{\mathbf{x}}
\newcommand{\bw}{\mathbf{w}}

\newcommand{\bb}{\mathbf{b}}
\newcommand{\bu}{\mathbf{u}}
\newcommand{\bv}{\mathbf{v}}
\newcommand{\bz}{\mathbf{z}}

\newcommand{\bh}{\mathbf{h}}

\newcommand{\bmu}{\boldsymbol{\mu}}
\newcommand{\balpha}{\boldsymbol{\alpha}}

\newcommand{\bxi}{\boldsymbol{\xi}}

\newcommand{\btheta}{{\boldsymbol{\theta}}}

\newcommand{\co}{{\cal O}}

\newcommand{\cd}{{\cal D}}

\newcommand{\cl}{{\cal L}}

\newcommand{\cu}{{\cal U}}

\newcommand{\cs}{{\cal S}}
\newcommand{\cn}{{\cal N}}

\DeclareMathOperator*{\sign}{sign}
\DeclareMathOperator*{\argmax}{argmax}

\DeclareMathOperator*{\cluster}{cluster}

\newcommand{\bbs}{{\mathbb S}}
\newcommand{\reals}{{\mathbb R}}

\newcommand{\zero}{{\mathbf{0}}}
\newcommand{\ndist}{\mathsf{N}}
\newcommand{\Dclust}{{\mathcal{D}_{\text{clusters}}}}

\DeclareMathOperator{\poly}{poly}

\newcommand{\inner}[1]{\langle #1 \rangle}
\newcommand{\binner}[1]{\biggl< #1 \biggr>}
\newcommand{\norm}[1]{\left\|#1\right\|}

\usepackage[style = alphabetic,citestyle=alphabetic,maxbibnames=99,sorting=nyt,natbib=true,backref=true]{biblatex}
\DefineBibliographyStrings{english}{%
  backrefpage = {Cited on page},
  backrefpages = {Cited on pages},
}

\newcommand{\note}[1]{\textcolor{red}{\textbf{#1}}}

\makeatletter
\newcommand{\printfnsymbol}[1]{%
  \textsuperscript{\@fnsymbol{#1}}%
}
\makeatother

\title{\textbf{The Double-Edged Sword of Implicit Bias: \\Generalization vs. Robustness in ReLU Networks}}
\author{
Spencer Frei\footnote{Equal contribution.}\phantom{$^\ast$}\\
UC Davis\\
\texttt{sfrei@ucdavis.edu}\\
\and 
Gal Vardi$^\ast$\\
TTI-Chicago and Hebrew University\\ 
\texttt{galvardi@ttic.edu}
\and
Peter L. Bartlett\\
UC Berkeley and Google\\
\texttt{peter@berkeley.edu}\\
\and
Nathan Srebro\\
TTI-Chicago\\
\texttt{nati@ttic.edu}
\and
{Collaboration on the Theoretical Foundations of Deep Learning (\href{https://deepfoundations.ai}{deepfoundations.ai})}
}

\begin{document}

\maketitle

\begin{abstract}
In this work, we study the implications of the implicit bias of gradient flow on generalization and adversarial robustness in ReLU networks.
    We focus on a setting where the data consists of clusters and the correlations between cluster means are small, and show that in two-layer ReLU networks gradient flow is biased towards solutions that generalize well, but are 
    vulnerable
    to adversarial examples. Our results hold even in cases where the network 
    is highly overparameterized.
    Despite the potential for 
    harmful 
    overfitting in such 
    settings, 
    we prove that the implicit bias of gradient flow prevents it. 
    However, the implicit bias also leads to non-robust solutions (susceptible to small adversarial $\ell_2$-perturbations), even though robust networks that fit the data exist.
\end{abstract}

\section{Introduction}

A central question in the theory of deep learning is how neural networks 
can 
generalize even when trained without any explicit regularization, and when there are more learnable parameters than training examples. In such optimization problems there are many solutions that label the training data correctly, and gradient descent seems to prefer solutions that generalize well \citep{zhang2017understanding}. Thus, it is believed that gradient descent induces an {\em implicit bias} towards solutions which enjoy favorable properties 
\citep{neyshabur2017exploring}.
Characterizing this bias in various settings has been a subject of extensive research in recent years, 
but
it is still not well understood when the implicit bias provably implies generalization in non-linear neural networks.

An additional intriguing phenomenon in deep learning 
is the abundance of \emph{adversarial examples} in trained neural networks.
In a seminal paper, \citet{szegedy2014intriguing} observed that deep networks are extremely vulnerable to adversarial examples,
namely, very small perturbations to the inputs can 
significantly 
change the predictions. This phenomenon has attracted considerable interest, and various attacks (e.g., \cite{goodfellow2015explaining,carlini2017adversarial,papernot2017practical,athalye2018obfuscated,carlini2018audio,wu2020making}) and defenses (e.g., \cite{papernot2016distillation,kurakin2017adversarial,madry2018towards,wong2018provable,croce2020reliable,wong2020fast}) were developed. 
However, the fundamental principles underlying the existence of adversarial examples are still unclear, and it is believed that for most tasks where trained neural networks suffer from a vulnerability to adversarial attacks, there should exist other neural networks which can be robust to such attacks.  This is suggestive of the possible role of the optimization algorithms used to train neural networks in the existence of adversarial examples. 
\ignore{
Despite a great deal of research, it is 
unclear why neural networks are so susceptible to adversarial examples. 
Specifically, it is not well-understood why gradient methods learn \emph{non-robust networks}, namely, networks that are susceptible to adversarial examples, even in cases where robust classifiers exist.
}

In this work, we study the implications of the implicit bias for generalization and robustness in ReLU networks, in a setting where the data 
consists of clusters 
(i.e., Gaussian mixture model) and the correlations between cluster means are small.
We
show that in two-layer ReLU networks trained with the logistic loss or the exponential loss, gradient flow is biased towards solutions that generalize well, albeit they are non-robust. 
Our results are independent of the network width, and hence they hold even where the network has significantly more parameters than training examples. In such an overparameterized setting, one might expect 
harmful 
overfitting to occur, but we prove that the implicit bias of gradient flow prevents it. On the flip side, in our setting the distances between clusters are large, and thus one might hope that gradient flow will converge to a robust network. However, we show that the implicit bias leads to non-robust solutions. 

Our results rely on known properties of the implicit bias in two-layer ReLU networks trained with the logistic or the exponential loss, which were shown by \citet{lyu2019gradient} and~\citet{ji2020directional}. They proved that if gradient flow in homogeneous models (which include two-layer ReLU networks) with such losses 
reaches
a small training loss, then it converges (in direction) to a KKT point of the maximum-margin problem in parameter space. We show that in clustered data distributions, with high probability over the training dataset, every network that satisfies the KKT conditions of the maximum-margin problem generalizes well  but is non-robust. 
Thus, instead of analyzing the trajectory of gradient flow directly in the complex setting of training two-layer ReLU networks, we demonstrate that investigating the KKT points is a powerful tool for understanding generalization and robustness.
We emphasize that our results hold in the \emph{rich} (i.e., \emph{feature learning}) regime, namely, the neural network training does not lie in the kernel regime, and thus we provide guarantees which go beyond the analysis achieved using NTK-based results. 

In a bit more detail, our main contributions are the following:
\begin{itemize}
	\item Suppose that the data distribution consists of $k$ clusters, and the training dataset is of size $n \geq \tilde{\Omega}(k)$. We show that with high probability over the size-$n$ dataset, if gradient flow achieves training loss smaller than $\frac{1}{n}$ at some time $t_0$, then it converges in direction to a network that generalizes well (i.e., has a small test error). Thus, gradient-flow-trained networks
 cannot 
 harmfully 
 overfit even if the network is highly overparameterized. The sample complexity $\tilde{\Omega}(k)$ in this result is optimal (up to log factors), since we cannot expect to perform well on unseen data using a training dataset that does not include at least one example from each cluster.
	\item In the same setting as above, we prove that gradient flow converges in direction to a non-robust network, even though there exist robust networks that classify the data correctly. 
	Specifically, we consider data distributions on $\reals^d$ such that the distance between every pair of clusters is $\Omega(\sqrt{d})$, and we show that there exists a two-layer ReLU network where flipping the output sign of a test example requires w.h.p. an $\ell_2$-perturbation of size $\Omega(\sqrt{d})$, but gradient flow converges to a network where we can flip the output sign of a test example with an $\ell_2$-perturbation of size much smaller than $\sqrt{d}$.
Moreover, the adversarial perturbation depends only on the data distribution, and not on the specific test example or trained neural network. Thus, the perturbation is both \emph{universal} \citep{moosavi2017universal,zhang2021survey} and \emph{transferable} \citep{liu2017delving,akhtar2018threat}. 
	We argue that clustered data distributions are a natural setting for analyzing the tendency of gradient methods to converge to non-robust solutions. Indeed, if positive and negative examples are not well-separated (i.e., the distances between points with opposite labels are small), then robust solutions do not exist. Thus, in order to understand the role of the optimization algorithm, we need a setting with sufficient separation between positive and negative examples.
\end{itemize}

The remainder of this paper is structured as follows: Below we discuss related work. In \secref{sec:prelim} we provide necessary notations and background, and introduce our setting and assumptions. In Sections~\ref{sec:generalization} and~\ref{sec:robustness} we state our main results on generalization and robustness (respectively), and provide some proof ideas, with all formal proofs deferred to the appendix. We conclude with a short discussion (\secref{sec:discussion}).

\subsection*{Related work}

\paragraph{Implicit bias in neural networks.}

The literature on implicit bias in neural networks has rapidly expanded in recent years, and cannot be reasonably surveyed here (see \citet{vardi2022implicit} for a survey).
In what follows, we discuss results that apply to two-layer ReLU or leaky-ReLU networks trained with gradient flow in classification settings.

By \citet{lyu2019gradient} and~\citet{ji2020directional}, homogeneous neural networks (and specifically two-layer ReLU networks, which are the focus of this paper) trained with exponentially-tailed classification losses converge in direction to a KKT point of the maximum-margin problem. 
Our analysis of the implicit bias relies on this result.
We note that the aforementioned KKT point may not be a global optimum of the maximum-margin problem \citep{vardi2022margin}.
Recently, \citet{kunin2022asymmetric} extended this result by showing bias towards margin maximization in a broader family of networks called \emph{quasi-homogeneous}.

\citet{lyu2021gradient,sarussi2021towards,frei2023implicit} studied implicit bias in two-layer leaky-ReLU networks with linearly-separable data, and proved that under some additional assumptions gradient flow converges to a linear classifier.
 \ignore{
studied the implicit bias in 
two-layer leaky ReLU networks trained on linearly separable and symmetric data, and showed that gradient flow converges to a linear classifier which maximizes the $\ell_2$ margin. Note that in our work we do not assume that the data is symmetric, but we assume that it is nearly orthogonal. Also, in our case we show that gradient flow might converge to a linear classifier that does not maximize the $\ell_2$ margin.
\citet{sarussi2021towards} studied gradient flow on two-layer leaky ReLU networks, where the training data is linearly separable. They showed convergence to a linear classifier based on an assumption called \emph{Neural Agreement Regime (NAR)}: starting from some time point, 
all positive neurons (i.e., neurons with a positive outgoing weight) agree on the classification of the training data, and similarly for the negative neurons.
However, it is unclear when this assumption holds a priori. 
}
\citet{chizat2020implicit} studied the dynamics of gradient flow on infinite-width homogeneous two-layer networks with exponentially-tailed losses, and showed bias towards margin maximization w.r.t. a certain function norm known as the variation norm.
\citet{phuong2020inductive} studied the implicit bias in two-layer ReLU networks trained on \emph{orthogonally separable data}. 

\citet{safran2022effective} proved implicit bias towards minimizing the number of linear regions in univariate two-layer ReLU networks, and used this result to obtain generalization bounds. Similarly to our work, they used the KKT conditions of the maximum-margin problem in parameter space to prove generalization in overparameterized networks. However, our setting is significantly different. Implications of the bias towards KKT points of the maximum-margin problem were also studied in \citet{haim2022reconstructing}, where they showed that this implicit bias can be used for reconstructing training data from trained ReLU networks.

\paragraph{Theoretical explanations for non-robustness in neural networks.}
Despite much research, the reasons for the abundance of adversarial examples in trained networks are still unclear \citep{goodfellow2015explaining,fawzi2018adversarial,shafahi2019adversarial,schmidt2018adversarially,khoury2018geometry,bubeck2019adversarial,allen2021feature,wang2020high,shah2020pitfalls,shamir2021dimpled,singla2021shift,wang2022adversarial,dohmatob2022non}. Below we discuss several prior theoretical works on this question.

In one line of work, it has been shown that 
small adversarial perturbations can be found for any fixed input in certain neural networks with random weights (drawn from the Gaussian distribution) 
\citep{daniely2020most,bubeck2021single,bartlett2021adversarial,montanari2022adversarial}.  These works differ in the assumptions about the width and depth of the networks as well as the activation functions considered.  However, since trained networks are non-random, these works are unable to 
capture the existence of adversarial examples in trained networks. 
\ignore{ 
 \citet{daniely2020most} showed
that small adversarial $\ell_2$-perturbations can be found in random ReLU networks where each layer has vanishing width relative to the previous layer. \citet{bubeck2021single} extended this result to two-layer neural networks without the vanishing width assumption, and \citet{bartlett2021adversarial} extended it to a large family of ReLU networks of constant depth. Finally, \citet{montanari2022adversarial} provided a similar result, but with weaker assumptions on the network width and activation functions. 
These works aim to explain the abundance of adversarial examples in neural networks, since they imply that adversarial examples are common in random networks, and specifically in random initializations of gradient-based methods. However, trained networks are clearly not random, and properties that hold in random networks may not hold in trained networks.
}

The result closest to ours was shown in \citet{vardi2022gradient}. Similarly to our result, they used the KKT conditions of the maximum-margin problem in parameter space, in order to prove that gradient flow converges to non-robust two-layer ReLU networks under certain assumptions. More precisely, they considered a setting where the training dataset $\cs$ consists of nearly-orthogonal points, and proved that every KKT point is non-robust w.r.t. $\cs$. Namely, for every two-layer network that satisfies the KKT conditions of the maximum-margin problem, and every point $\bx_i$ from $\cs$, it is possible to flip the output's sign with a small perturbation. Their result has two main limitations: (1) It considers robustness w.r.t. the training data, while the more common setting in the literature considers robustness w.r.t. test data, as it is often more crucial to avoid adversarial perturbations in test examples; (2) Since they assume near orthogonality of the training data, the size of the dataset $\cs$ must be smaller than the input dimension.\footnote{They also give a version of their result, where instead of assuming this upper bound on the size of the dataset, they assume an upper bound on the number of points that attain the margin in the trained network, but it is not clear a priori when this assumption is likely to hold.} 
Thus, they considered a \emph{high dimensional} setting. We note that high-dimensional settings often have a different generalization behavior than low-dimensional settings (e.g., overfitting can be \emph{benign} in the high-dimensional setting, but harmful in a low-dimensional setting \citep{kornowski2023tempered}). 
Our result does not suffer from these limitations, since we consider robustness w.r.t. test data, and the size of our training dataset might be very large. In our results, we essentially require near orthogonality of the cluster means, as opposed to near orthogonality of the training dataset in their result.

Finally, in \citet{bubeck2021law} and \citet{bubeck2021universal}, the authors proved (under certain assumptions) that overparameterization is necessary if one wants to interpolate training data using a neural network with a small Lipschitz constant. 
Namely, neural networks with a small number of parameters are not expressive enough to interpolate the training data while having a small Lipschitz constant.
These results suggest that overparameterization might be necessary for robustness. In this work, we show that even if the network is highly overparameterized, the implicit bias of the optimization method 
can 
prevent convergence to robust solutions.     

\ignore{
\paragraph{Feature learning.}

This work considers generalization and robustness in neural networks trained in the feature-learning regime.
Several works in recent years analyzed the trajectory when training neural networks with gradient methods 
in this regime, in settings different from ours. These works include \cite{brutzkus2017sgd,frei2021provable,lyu2021gradient,sarussi2021towards,frei2022rfa}
\note{TODO} 
}

\section{Preliminaries} \label{sec:prelim}

We use bold-face letters to denote vectors, e.g., $\bx=(x_1,\ldots,x_d)$. For $\bx \in \reals^d$ we denote by $\norm{\bx}$ the Euclidean norm.
We denote by $\onefunc[\cdot]$ the indicator function, for example $\onefunc[t \geq 5]$ equals $1$ if $t \geq 5$ and $0$ otherwise. We denote $\sign(z) = 1$ if $z > 0$ and $-1$ otherwise.
For an integer $d \geq 1$ we denote $[d]=\{1,\ldots,d\}$. For a set $A$ we denote by $\cu(A)$ the uniform distribution over $A$. We denote by $\ndist(\mu,\sigma^2)$ the normal distribution with mean $\mu \in \reals$ and variance $\sigma^2$, and by $\ndist(\bmu,\Sigma)$ the multivariate normal distribution with mean $\bmu$ and covariance matrix $\Sigma$. The identity matrix of size $d$ is denoted by $I_d$.
We use standard asymptotic notation $\co(\cdot)$ and $\Omega(\cdot)$ 
to hide constant factors, and $\tilde{\co}(\cdot), \tilde{\Omega}(\cdot)$ to hide logarithmic factors.
We use $\log$ for the logarithm with base $2$ and $\ln$ for the natural logarithm.

In this work, we consider depth-$2$ ReLU neural networks. The ReLU activation function is defined by $\phi(z) = \max\{0,z\}$.  Formally, a depth-$2$ network $\cn_\btheta$ of width $m$ is parameterized by $\btheta = [\bw_1,\ldots,\bw_m, \bb, \bv]$ where $\bw_i \in \reals^d$ for all $i \in [m]$ and $\bb,\bv \in \reals^m$, and for every input $\bx \in \reals^d$ we have 
\[
	\cn_\btheta(\bx) = \sum_{j \in [m]}v_j \phi(\bw_j^\top \bx + b_j)~. 
\]
We sometimes view $\btheta$ as the vector obtained by concatenating the vectors $\bw_1,\ldots,\bw_m, \bb, \bv$. Thus, $\norm{\btheta}$ denotes the $\ell_2$ norm of the vector $\btheta$. We note that in this work we train both layers of the ReLU network.

We denote $\Phi(\btheta; \bx) := \cn_\btheta(\bx)$.
We say that a network is \emph{homogeneous} if there exists $L>0$ such that for every $\alpha>0$ and $\btheta,\bx$ we have $\Phi(\alpha \btheta; \bx) = \alpha^L \Phi(\btheta; \bx)$. Note that depth-$2$ ReLU networks as defined above are homogeneous (with $L=2$).

We next define gradient flow and remind the reader of some recent results on the implicit bias of gradient flow in two-layer ReLU networks. 
Let $\cs = \{(\bx_i,y_i)\}_{i=1}^n \subseteq \reals^d \times \{-1,1\}$ be a binary classification training dataset. Let $\Phi(\btheta; \cdot):\reals^d \to \reals$ be a neural network parameterized by $\btheta$. 
For a loss function $\ell:\reals \to \reals$ the \emph{empirical loss} of $\Phi(\btheta; \cdot)$ on the dataset $\cs$ is 
\begin{equation}
\label{eq:objective}
	\cl(\btheta) := \frac{1}{n} \sum_{i=1}^n \ell(y_i \Phi(\btheta; \bx_i))~.
\end{equation} 
We focus on the exponential loss $\ell(q) = e^{-q}$ and the logistic loss $\ell(q) = \log(1+e^{-q})$.

We consider gradient flow on the objective given in \eqref{eq:objective}. This setting captures the behavior of gradient descent with an infinitesimally small step size. Let $\btheta(t)$ be the trajectory of gradient flow. Starting from an initial point $\btheta(0)$, the dynamics of $\btheta(t)$ is given by the differential equation 
$\frac{d \btheta(t)}{dt} \in -\partial^\circ \cl(\btheta(t))$. Here, $\partial^\circ$ denotes the \emph{Clarke subdifferential} \citep{clarke2008nonsmooth}, which is a generalization of the derivative for non-differentiable functions. 

We now remind the reader of a recent result concerning the implicit bias of gradient flow over the exponential and logistic losses for homogeneous neural networks.  Note that since homogeneous networks satisfy $\sign(\Phi(\alpha \btheta; \bx)) = \sign(\Phi(\btheta;\bx))$ for any $\alpha>0$, the sign of the network output of homogeneous networks depends only on the direction of the parameters $\btheta$.  The following theorem provides a characterization of the implicit bias of gradient flow by showing that the trajectory of the weights $\btheta(t)$ \textit{converge in direction} to a first-order stationary point of a particular constrained optimization problem, where $\btheta$ \textit{converges in direction} to $\tilde \btheta$ means
$\lim_{t \to \infty}\frac{\btheta(t)}{\norm{\btheta(t)}} = 
\frac{\tilde{\btheta}}{\norm{\tilde{\btheta}}}$.  Note that since ReLU networks are non-smooth, the first-order stationarity conditions (i.e., the Karush--Kuhn--Tucker conditions, or KKT conditions for short) are defined using the Clarke subdifferential
(see \citet{lyu2019gradient} and \citet{dutta2013approximate} for more details on the KKT conditions in non-smooth optimization problems).

\begin{theorem}[Paraphrased from \citet{lyu2019gradient,ji2020directional}] \label{thm:known KKT}
    Let $\Phi(\btheta; \cdot)$ be a homogeneous ReLU neural network parameterized by $\btheta$. Consider minimizing either the exponential or the logistic loss over a binary classification dataset $ \{(\bx_i,y_i)\}_{i=1}^n$ using gradient flow. Assume that there exists time $t_0$ such that $\cl(\btheta(t_0))<\frac{1}{n}$ 
    (and thus $y_i \Phi(\btheta(t_0); \bx_i) > 0$ for every $\bx_i$). 
    Then, gradient flow converges in direction to a first-order stationary point (KKT point) of the following maximum margin problem in parameter space:
\begin{equation}
\label{eq:optimization problem}
	\min_\btheta \frac{1}{2} \norm{\btheta}^2 \;\;\;\; \text{s.t. } \;\;\; \forall i \in [n] \;\; y_i \Phi(\btheta; \bx_i) \geq 1~.
\end{equation}
Moreover, $\cl(\btheta(t)) \to 0$ and $\norm{\btheta(t)} \to \infty$ as $t \to \infty$.
\end{theorem}

Theorem~\ref{thm:known KKT} gives a characterization of the implicit bias of gradient flow with the exponential and the logistic loss for homogeneous ReLU networks.  Note that the theorem makes no assumption on the initialization, training data, or number of parameters in the network; the only requirement is that the network is homogeneous and that at some time point in the gradient flow trajectory, the network is able to achieve small training loss.  The theorem shows that although there are many ways to configure the network parameters to achieve small training loss (via overparameterization), gradient flow only converges (in direction) to networks which satisfy the KKT conditions of Problem~(\ref{eq:optimization problem}).  It is important to note that satisfaction of the KKT conditions is not sufficient for global optimality of the constrained optimization problem~\citep{vardi2022margin}.  We further note that if the training data are sampled i.i.d. from a distribution with label noise (e.g., a class-conditional Gaussian mixture model, or a distribution where labels $y_i$ are flipped to $-y_i$ with some nonzero probability), networks which have parameters that are feasible w.r.t. the constraints of Problem~(\ref{eq:optimization problem}) have overfit to noise, and understanding the generalization behavior of even globally optimal solutions to Problem~(\ref{eq:optimization problem}) in this setting is the subject of significant research~\citep{montanari2019maxmargin,chatterji2020linearnoise,frei2023benign}.

Finally, we introduce the distributional setting that we consider.  We consider a distribution $\Dclust$ on $\reals^d \times \{-1,1\}$ that consists of $k$ clusters with means $\bmu^{(1)},\ldots,\bmu^{(k)} \in \reals^d$ and covariance $\sigma^2 I_d$ (i.e., a Gaussian mixture model), such that the examples in the $j$-th cluster are labeled by $y^{(j)} \in \{-1,1\}$. More formally, $(\bx,y) \sim \Dclust$ is generated as follows: we draw $j \sim \cu([k])$ and $\bx \sim \ndist(\bmu^{(j)}, \sigma^2 I_d)$, and set $y=y^{(j)}$. We assume that there exist $i,j \in [k]$ with $y^{(i)} \neq y^{(j)}$. 
Moreover, we assume the following: 
\begin{assumption} \label{ass:dist}
    We have:
    \begin{itemize}
    	\item $\norm{\bmu^{(j)}} = \sqrt{d}$ for all $j \in [k]$.
    	\item 
    		$0< \sigma \leq 1$.
    	\item $k \left( \max_{i \neq j} |\inner{\bmu^{(i)}, \bmu^{(j)}}| + 4\sigma \sqrt{d} \ln(d) + 1 \right) \leq \frac{d  - 4\sigma \sqrt{d} \ln(d) + 1 }{10}$.
    \end{itemize}
\end{assumption}

\begin{example} \label{ex:dist}
    Below we provide simple examples of settings that satisfy the assumption:
    \begin{itemize}
        \item Suppose that the cluster means satisfy $|\inner{\bmu^{(i)}, \bmu^{(j)}} | = \tilde{\co}(\sqrt{d})$ for every $i \neq j$. This condition holds, e.g., if we choose each cluster mean i.i.d. from the uniform distribution on the sphere $\sqrt{d} \cdot \bbs^{d-1}$ (see, e.g., \citet[Lemma 3.1]{vardi2022gradient}). Let $\sigma=1$, namely, each cluster has a radius of roughly $\sqrt{d}$. Then, the assumption can be satisfied by choosing $k = \tilde{\co}(\sqrt{d})$.
        \item Suppose that the cluster means are exactly orthogonal (i.e., $\inner{\bmu^{(i)}, \bmu^{(j)}} = 0$ for all $i\neq j$), and $\sigma=1/\sqrt{d}$. Then, the assumption can be satisfied by choosing $k = \tilde{\co}(d)$.
        \item If the number of clusters is $k=\tilde{\co}(1)$, then the assumption may hold even where $\max_{i \neq j} |\inner{\bmu^{(i)}, \bmu^{(j)}}| = \tilde{\Theta}(d)$ (for any $0 < \sigma \leq 1)$.
    \end{itemize}
\end{example}

A few remarks are in order.
First, the assumption that $\norm{\bmu^{(j)}}$ is exactly $\sqrt{d}$ is for convenience, and we note that it may be relaxed (to have all cluster means approximately of the same norm) without affecting our results significantly. Note that in the case where $\sigma=1$, the radius of each cluster is roughly of the same magnitude as the cluster mean.
Second, we assume for convenience that the noise (i.e., the deviation from the cluster's mean) is drawn from a Gaussian distribution with covariance matrix $\sigma^2 I_d$. However, we note that this assumption can be generalized to any distribution $\cd_{\text{noise}}$ such that for every unit vector $\be$ the noise $\bxi \sim \cd_{\text{noise}}$ satisfies w.h.p. that $\inner{\bxi,\be} = \tilde{\co}(1)$ and $\norm{\bxi} = \tilde{\co}(\sqrt{d})$. This property holds, e.g., for a $d$-dimensional Gaussian distribution $\mathsf{N}(0,\Sigma)$, where $\mathrm{tr}[\Sigma] = d$ and $\norm{\Sigma}_2 = O(1)$ (see~\citet[Lemma 3.3]{frei2023implicit}), and more generally for a class of sub-Gaussian distributions (see~\citet[Claim 3.1]{hu2020surprising}). 
Third, note that the third part of \assref{ass:dist} essentially requires that the number of clusters $k$ cannot be too large and the correlations between cluster means cannot be too large.
Finally, we remark that when $k$ is small, our results may be extended to the case where $\sigma>1$. For example, if $k = \tilde{\co}(1)$ and $\max_{i \neq j} |\inner{\bmu^{(i)}, \bmu^{(j)}}| = \tilde{\co}(\sqrt{d})$, our generalization result (\thmref{thm:generalizaton}) can be extended to the case where $\sigma = \tilde{\co}(d^{1/8})$. We preferred to avoid handling $\sigma>1$ in order to simplify the proofs.  

Moreover, it is worth noting that \assref{ass:dist} implies that the data is w.h.p. linearly separable (see \lemref{lem:linearly separable} below, and a proof in \appref{app:proof of linearly separable}). However, in this work we consider learning using overparameterized ReLU networks, and it is not obvious a priori that gradient methods do not 
harmfully 
overfit in this case.  
Indeed, it has been shown that ReLU networks trained by gradient descent can interpolate training data and fail to generalize well in some distributional settings~\citep{kou2023benign}. 

\begin{lemma} \label{lem:linearly separable}
    Let $\bu = \sum_{q \in [k]} y^{(q)} \bmu^{(q)}$. Then, with probability at least $1-2d^{1-\ln(d)/2} = 1-o_d(1)$ over $(\bx,y) \sim \Dclust$, we have $y = \sign(\bu^\top \bx)$. 
\end{lemma}

\section{Generalization} \label{sec:generalization}

In this section, we show that under our assumptions on the distribution $\Dclust$, gradient flow does not 
harmfully 
overfit. Namely, even if the learned network is highly overparameterized, the implicit bias of gradient flow guarantees convergence to a solution that generalizes well. Moreover, we show that the sample complexity is optimal.
The main result of this section is stated in the following theorem:

\begin{theorem} \label{thm:generalizaton}
        Let $\epsilon, \delta \in (0,1)$.
	Let $\cs = \{(\bx_i,y_i)\}_{i=1}^n \subseteq \reals^d \times \{-1,1\}$ be a training set drawn i.i.d. from the distribution $\Dclust$, where $n \geq k \ln^2(d)$. Let $\cn_\btheta$ be a depth-$2$ ReLU network such that $\btheta = [\bw_1,\ldots,\bw_m,\bb,\bv]$ is a KKT point of Problem~(\ref{eq:optimization problem}). 
        Provided $d$ is sufficiently large such that 
        $\delta^{-1} \leq \frac{1}{3} d^{\ln(d) - 1}$
        and 
        \[
            n \leq \min\left\{ \sqrt{\frac{\delta}{3}} \cdot e^{d/32}, \frac{\sqrt{\delta}}{3} \cdot d^{\ln(d)/4}, \frac{\epsilon}{4} \cdot d^{\ln(d)/2} \right\}~,
        \]
        with probability at least $1-\delta$ over $\cs$, we have 
        \[
            \Pr_{(\bx,y) \sim \Dclust} \left[ y \cn_\btheta(\bx) \leq 0 \right] \leq \epsilon~.
        \]
\end{theorem}

\ignore{
\begin{theorem} \label{thm:generalizaton}
	Let $\cs = \{(\bx_i,y_i)\}_{i=1}^n \subseteq \reals^d \times \{-1,1\}$ be a training set drawn i.i.d. from the distribution $\Dclust$, where $n \geq k \ln^2(d)$ and is at most polynomial in $d$. Let $\cn_\btheta$ be a depth-$2$ ReLU network such that $\btheta = [\bw_1,\ldots,\bw_m,\bb,\bv]$ is a KKT point of Problem~(\ref{eq:optimization problem}). 
	Then, with probability at least 
	\[
		1- \left(3 n^2 d^{-\frac{\ln(d)}{2}} + n^2 e^{-d/16} + d^{1-\ln(d)} \right) 
		\geq 1 - d^{-\omega_d(1)}
	\]	
	over $\cs$, 
	we have 
	\[
		\Pr_{(\bx,y) \sim \Dclust} \left[ y \cn_\btheta(\bx)) \leq 0 \right] 
		\leq 4 n d^{-\frac{\ln(d)}{2}}  
		\leq d^{-\omega_d(1)}~.
	\]
\end{theorem}
}

The sample complexity requirement in \thmref{thm:generalizaton} is $n = \tilde{\Omega}(k)$. Essentially, it requires that the dataset $\cs$ will include at least one example from each cluster. 
Clearly, 
any learning algorithm cannot perform well on unseen clusters.
Hence the sample complexity requirement in 
the theorem
is tight (up to log factors).

The assumptions in \thmref{thm:generalizaton} include upper bounds on $\delta^{-1}$ and $n$. Note that the expressions in these upper bounds are super-polynomial in $d$, and in particular if $n,\delta^{-1},\epsilon^{-1} = \poly(d)$, then these assumptions hold for a sufficiently large $d$.  
Admittedly, enforcing an upper bound on the training dataset's size is uncommon in generalization results.  However, if $n$ is exponential in $d$, it is not hard to see that there will be clusters which have both positive and negative examples within radius $\sigma$ of the cluster center, essentially introducing a form of label noise to the problem.  Since KKT points of Problem~(\ref{eq:optimization problem}) interpolate the training data, this would imply that the network has interpolated training data with label noise---in other words, it has `overfit' to noise.  Understanding the generalization behavior of interpolating neural networks in the presence of label noise is a very technically challenging problem for which much is unknown, especially if one seeks to understand this by only relying upon the properties of KKT conditions for margin maximization.  It is noteworthy that all existing non-vacuous generalization bounds for interpolating nonlinear neural networks in the presence of label noise require $n<d$~\citep{frei2022benign,cao2022convolutional,xu2023benign,frei2023benign,kou2023benign}.

Combining \thmref{thm:generalizaton} with \thmref{thm:known KKT}, we conclude that w.h.p. over a training dataset of size 
$n \geq k \ln^2(d)$ (and under some additional mild requirements),
if gradient flow reaches 
empirical loss 
smaller than $\frac{1}{n}$,
then it converges in direction to a neural network that generalizes well. This result is width-independent, thus, it holds irrespective of the network width. Specifically, even if the network is highly overparameterized, the implicit bias of gradient flow prevents 
harmful 
overfitting. Moreover, the result does not depend directly on the initialization of gradient flow. That is, it holds whenever gradient flow reaches small empirical loss after some finite time.
Thus, by relying on the KKT conditions of the max-margin problem instead of analyzing the full gradient flow trajectory, we can prove generalization without the need to prove convergence. 

\subsection{Proof idea} \label{subsec:generalization ideas}

The proof of \thmref{thm:generalizaton} is given in \appref{app:generalization}. Here we discuss the high-level approach.
Let $\btheta = [\bw_1,\ldots,\bw_m,\bb,\bv]$ be a KKT point of Problem~(\ref{eq:optimization problem}). Thus, we have $\cn_\btheta(\bx) = \sum_{j \in [m]} v_j \phi(\bw_j^\top \bx + b_j)$.
Since $\btheta$ satisfies the KKT conditions of Problem~(\ref{eq:optimization problem}), then there are $\lambda_1,\ldots,\lambda_n$ such that for every $j \in [m]$ we have
\begin{equation}
\label{eq:kkt condition w idea}
	\bw_j = \sum_{i \in [n]} \lambda_i \nabla_{\bw_j} \left( y_i \cn_{\btheta}(\bx_i) \right) =  \sum_{i \in [n]} \lambda_i y_i v_j \phi'_{i,j} \bx_i~,
\end{equation}
where $\phi'_{i,j}$ is a subgradient of $\phi$ at $\bw_j^\top \bx_i + b_j$, i.e., if $\bw_j^\top \bx_i + b_j \neq 0$ then $\phi'_{i,j} = \onefunc[\bw_j^\top \bx_i + b_j \geq 0]$, and otherwise $\phi'_{i,j}$ is some value in $[0,1]$. Also we have $\lambda_i \geq 0$ for all $i$, and $\lambda_i=0$ if 
$y_i  \cn_{\btheta}(\bx_i) 
\neq 1$. Likewise, we have
\begin{equation}
\label{eq:kkt condition b idea}
	b_j = \sum_{i \in [n]} \lambda_i \nabla_{b_j} \left( y_i \cn_{\btheta}(\bx_i) \right) =  \sum_{i \in [n]} \lambda_i y_i v_j \phi'_{i,j}~.
\end{equation}

In the proof, using a careful analysis of Eq.~(\ref{eq:kkt condition w idea})~and~(\ref{eq:kkt condition b idea}) we show that w.h.p. $\cn_\btheta$ classifies correctly a fresh example. 
Note that by \eqref{eq:kkt condition w idea}, each neuron $j \in [m]$ is a linear combination of the inputs $\bx_1,\ldots,\bx_n$ with coefficients $\lambda_i y_i v_j \phi'_{i,j}$. 
Let $r \in [k]$ be a cluster with $y^{(r)}=1$ (the argument for $y^{(r)}=-1$ is similar). Intuitively, we show that 
the sum of the aforementioned coefficients over all inputs in cluster $r$ and over all positive neurons (i.e., neurons with $v_j>0$) is large. Namely, the positive neurons have a large component in the direction of cluster $r$. Then, using the near-orthogonality of the clusters, we show that given a fresh example from cluster $r$, the positive contribution of the component in the direction of cluster $r$ is sufficiently large, such that the output of $\cn_\btheta$ is positive.

More precisely, the main argument can be described as follows.
We denote $J := [m]$, $J_+ := \{j \in J: v_j > 0\}$, and $J_- := \{j \in J: v_j < 0\}$. Moreover, we denote $I := [n]$ and $Q := [k]$.
For 
$q \in Q$ 
we denote $I^{(q)} = \{i \in I: \bx_i \text{ is in cluster } q\}$.
Consider the network's output for an input $\bx$ from cluster $r \in Q$ with $y^{(r)}=1$. We have
\begin{equation} \label{eq:network.output.idea}
    \cn_\btheta(\bx) =
    \sum_{j\in J_+} v_j \phi(\bw_j^\top \bx + b_j) + \sum_{j\in J_-} v_j \phi(\bw_j^\top \bx + b_j) \geq \sum_{j\in J_+} v_j (\bw_j^\top \bx + b_j) + \sum_{j\in J_-} v_j \phi(\bw_j^\top \bx + b_j), 
\end{equation}
where we used $\phi(z) \geq z$ for all $z\in \mathbb R$.  This suggests the following possibility: if we can ensure that $\sum_{j\in J_+} v_j (\bw_j^\top \bx + b_j)$ is large and positive while $\sum_{j\in J_-} v_j \phi(\bw_j^\top \bx+b_j)$ is not too negative, then the network will accurately classify the example $\bx$. Using 
Eq.~(\ref{eq:kkt condition w idea})~and~(\ref{eq:kkt condition b idea})
and that 
$y^{(r)} = 1$
(so $y_i=1$ for $i\in I^{(r)}$),  the first term in the above decomposition is equal to
\begin{align*}
    \sum_{j\in J_+} v_j (\bw_j^\top \bx + b_j) &= \sum_{j \in J_+} v_j \left[ \sum_{i\in I}  \lambda_i y_i v_j\phi'_{i,j}(\bx_i^\top \bx+1) \right]
    \\
    &= \sum_{j \in J_+} \left[ \left( \sum_{i \in I^{(r)}} \lambda_i v_j^2 \phi'_{i,j} (\bx_i^\top \bx + 1) \right) + \sum_{q \in Q \setminus \{r\}} \sum_{i \in I^{(q)}} \lambda_i y_i v_j^2 \phi'_{i,j} (\bx_i^\top \bx + 1) \right]
    \\
    &\geq   \left( \sum_{i \in I^{(r)}} \sum_{j \in J_+} \lambda_i v_j^2 \phi'_{i,j} (\bx_i^\top \bx + 1) \right) - \sum_{q \in Q \setminus \{r\}} \sum_{i \in I^{(q)}} \sum_{j \in J_+} \lambda_i v_j^2 \phi'_{i,j} | \bx_i^\top \bx + 1|~.
\end{align*}
Since $\bx$ comes from cluster $r$ and the clusters are nearly orthogonal, the pairwise correlations $\bx_i^\top \bx$ will be large and positive when $i\in I^{(r)}$ but will be small in magnitude when 
$i \in I^{(q)}$ for $q \neq r$.
Thus, we can hope that this term will be large and positive if we can show that the quantity $\sum_{i\in I^{(r)}} \sum_{j\in J^+} \lambda_i v_j^2 \phi'_{i,j}$ is not too small relative to the quantity 
$\sum_{q \in Q \setminus \{r\}} \sum_{i \in I^{(q)}} \sum_{j \in J_+} \lambda_i v_j^2 \phi'_{i,j}$.
By similar arguments, in order to show the second term in~\eqref{eq:network.output.idea} is not too negative, we need to understand how the quantity $\sum_{i\in I^{(q)}} \sum_{j\in J_-}  \lambda_i v_j^2 \phi'_{i,j}$ varies across different clusters $q \in Q$. 
Hence, in the proof we analyze how the quantities
\[ \sum_{i\in I^{(q)}} \sum_{j\in J_+} \lambda_i v_j^2 \phi'_{i,j}, \quad \sum_{i\in I^{(q)}} \sum_{j\in J_-} \lambda_i v_j^2 \phi'_{i,j}\]
relate to each other for different clusters $q \in Q$, and show that these quantities are all of the same order. Then, we conclude that w.h.p. $\bx$ is classified correctly. 

\ignore{
In the proof, we show using a careful analysis of Eq.~(\ref{eq:kkt condition w idea})~and~(\ref{eq:kkt condition b idea}) that w.h.p. $\cn_\btheta$ classifies correctly a fresh example. 
Note that by \eqref{eq:kkt condition w idea}, each neuron $j \in [m]$ is a linear combination of the inputs $\bx_1,\ldots,\bx_n$ with coefficients $\alpha_{i,j} := \lambda_i y_i v_j \phi'_{i,j}$. 
Let $q \in [k]$ be a cluster with $y^{(q)}=1$ (the argument for $y^{(q)}=-1$ is similar). Intuitively, we show that the near orthogonality of the clusters implies that the sum of the coefficients $\alpha_{i,j}$ over all inputs in cluster $q$ and over all positive neurons (i.e., neurons with $v_j>0$) is large\footnote{More precisely, we show that the sum of $v_j \alpha_{i,j}$ (over all inputs in cluster $q$ and all positive neurons) is large.}. Namely, the positive neurons have a large component in the direction of cluster $q$. Then, using the near-orthogonality of the clusters, we show that given a fresh example from cluster $q$, the positive contribution of the component in the direction of cluster $q$ is sufficiently large, such that the output of $\cn_\btheta$ is positive.
}

\section{Robustness} \label{sec:robustness}

We begin by introducing the definition of $R(\cdot)$-robustness.
\begin{definition}\label{def:robustness}
    Given some function $R(\cdot)$, we say that a neural network $\cn_\btheta$ is \emph{$R(d)$-robust} w.r.t. a distribution $\cd_\bx$ over $\reals^d$ if for every $r = o(R(d))$, with probability $1-o_d(1)$ over $\bx \sim \cd_\bx$, for every $\bx' \in \reals^d$ with $\norm{\bx - \bx'} \leq r$ we have $\sign(\cn_\btheta(\bx')) = \sign(\cn_\btheta(\bx))$.
\end{definition}
Thus, a neural net $\cn_\btheta$ is $R(d)$-robust if changing the label of an example cannot be done with a perturbation of size $o(R(d))$.
Note that we consider here $\ell_2$ perturbations.

\ignore{ 
As we mentioned in Section~\ref{sec:prelim}, since for a homogeneous $\Phi$ we have
$\sign\left( \Phi(\alpha \btheta; \bx) \right) = \sign\left( \Phi(\btheta; \bx)\right)$ for every $\alpha>0$ and $\bx$, the robustness of the network depends only on the direction of $\frac{\btheta}{\norm{\btheta}}$, and does not depend on the scale of $\btheta$. 
}

For the distribution $\Dclust$ under consideration, it is straightforward to show that classifiers cannot be $R(d)$-robust if $R(d) = \omega(\sqrt d)$: since the distance between examples in different clusters is w.h.p. $\co(\sqrt d)$, it is clearly possible to flip the sign of an example with a perturbation of size $\co(\sqrt d)$.  In particular, the best we can hope for is $\sqrt d$-robustness.  In the following theorem, we show that there exist two-layer ReLU networks which can both achieve small test error and the optimal level of $\sqrt d$-robustness. 

\ignore{
Consider the distribution $\Dclust$, and let $\cd_\bx$ be the marginal distribution on $\reals^d$. Thus, $\cd_\bx$ is a Gaussian mixture.  
In the distribution $\cd_\bx$, the distance between two inputs is w.h.p. $\co(\sqrt{d})$, and therefore in any network with small error w.r.t. $\Dclust$, a perturbation of size $\co(\sqrt{d})$ clearly suffices for flipping the sign of the output. Hence, the best we can hope for is $\sqrt{d}$-robustness. We will now show that a $\sqrt{d}$-robust network indeed exists, but gradient flow converges to a network where we can flip the sign of the output w.h.p. with perturbations of size much smaller than $\sqrt{d}$ (and hence it is not $\sqrt{d}$-robust).

We start by showing that there exists a robust network with a small test error for the distribution $\Dclust$:
}

\begin{theorem} \label{thm:robust exists}
    For every $r \geq k$, there exists a depth-$2$ ReLU network $\cn:\reals^d \to \reals$ of width $r$ such that for $(\bx,y) \sim \Dclust$, with probability at least $1-d^{-\omega_d(1)}$ we have $y\cn(\bx) \geq 1$, and flipping the sign of the output requires a perturbation of size larger than $\frac{\sqrt{d}}{8}$ (for a sufficiently large $d$). Thus, $\cn$ classifies the data correctly w.h.p., and it is $\sqrt{d}$-robust w.r.t. $\cd_\bx$.
\end{theorem}
 
Thus, we see that $\sqrt d$-robust networks exist.  In the following theorem, we show that the implicit bias of gradient flow constrains the level of robustness of \textit{trained} networks whenever the number of clusters $k$ is large. 

\begin{theorem} \label{thm:non-robust}
        Let $\epsilon, \delta \in (0,1)$.
	Let $\cs = \{(\bx_i,y_i)\}_{i=1}^n \subseteq \reals^d \times \{-1,1\}$ be a training set drawn i.i.d. from the distribution $\Dclust$, where $n \geq k \ln^2(d)$. 
	We denote $Q_+ = \{q \in [k]: y^{(q)} = 1\}$ and $Q_- = \{q \in [k]: y^{(q)} = -1\}$, and assume that $\min\left\{\frac{|Q_+|}{k},\frac{|Q_-|}{k} \right\} \geq c$ for some $c > 0$.
	Let $\cn_\btheta$ be a depth-$2$ ReLU network such that $\btheta = [\bw_1,\ldots,\bw_m,\bb,\bv]$ is a KKT point of Problem~(\ref{eq:optimization problem}). 
	Provided $d$ is sufficiently large such that
        $\delta^{-1} \leq \frac{1}{3} d^{\ln(d) - 1}$
        and 
        \[
            n \leq \min\left\{ \sqrt{\frac{\delta}{3}} \cdot e^{d/32}, \frac{\sqrt{\delta}}{3} \cdot d^{\ln(d)/4}, \frac{\epsilon}{4} \cdot d^{\ln(d)/2} \right\}~,
        \]
        with probability at least $1-\delta$ over $\cs$, there is a vector $\bz = \eta \cdot \sum_{j \in [k]} y^{(j)} \bmu^{(j)}$ with $\eta > 0$ and $\norm{\bz} \leq \co\left( \sqrt{\frac{d}{c^2 k}} \right)$, such that 
	\[
		\Pr_{(\bx,y) \sim \Dclust} \left[ \sign(\cn_\btheta(\bx)) \neq \sign(\cn_\btheta(\bx - y \bz)) \right] 
		\geq 1 - \epsilon~.
	\]
\end{theorem}

\ignore{
\begin{theorem} \label{thm:non-robust}
	Let $\cs = \{(\bx_i,y_i)\}_{i=1}^n \subseteq \reals^d \times \{-1,1\}$ be a training set drawn i.i.d. from the distribution $\Dclust$, where $n \geq k \ln^2(d)$ and is at most polynomial in $d$. 
	We denote $Q_+ = \{q \in [k]: y^{(q)} = 1\}$ and $Q_- = \{q \in [k]: y^{(q)} = -1\}$, and assume that $\min\left\{\frac{|Q_+|}{k},\frac{|Q_-|}{k} \right\} \geq c$ for some $c > 0$.
	Let $\cn_\btheta$ be a depth-$2$ ReLU network such that $\btheta = [\bw_1,\ldots,\bw_m,\bb,\bv]$ is a KKT point of Problem~(\ref{eq:optimization problem}). 
	Then, 
        with probability at least
	\[
		1- \left(3 n^2 d^{-\frac{\ln(d)}{2}} + n^2 e^{-d/16} + d^{1-\ln(d)} \right) 
		\geq 1 - d^{-\omega_d(1)} 
	\]	
	over $\cs$, there is a vector $\bz = \eta \cdot \sum_{j \in [k]} y^{(j)} \bmu^{(j)}$ with $\eta > 0$ and $\norm{\bz} \leq \co\left( \sqrt{\frac{d}{c^2 k}} \right)$, such that 
	\[
		\Pr_{(\bx,y) \sim \Dclust} \left[ \sign(\cn_\btheta(\bx)) \neq \sign(\cn_\btheta(\bx - y \bz)) \right] 
		\geq 1 - 4 n d^{-\frac{\ln(d)}{2}}
		\geq 1 - d^{-\omega_d(1)}~.
	\]
\end{theorem}
}

Note that the expressions in the upper bounds on $n$ and $\delta^{-1}$ are super-polynomial in $d$, and hence these requirements are mild (e.g., they hold for a sufficiently large $d$ when $n,\delta^{-1},\epsilon^{-1} = \poly(d)$).  As we mentioned in the discussion following \thmref{thm:generalizaton}, we believe removing the requirement for an upper bound on $n$ would be highly nontrivial.  

\thmref{thm:non-robust} implies that if $c^2 k=\omega_d(1)$, then w.h.p. over the training dataset, every KKT point of Problem~(\ref{eq:optimization problem}) is not $\sqrt{d}$-robust. 
Specifically, if $c$ is constant, namely, at least a constant fraction of the clusters have positive labels and a constant fraction of the clusters have negative labels, then the network is not $\sqrt{d}$-robust if $k=\omega_d(1)$.
Recall that by \thmref{thm:generalizaton}, we also have w.h.p. that every KKT point generalizes well. Overall, combining Theorems~\ref{thm:known KKT},~\ref{thm:generalizaton},~\ref{thm:robust exists}, and~\ref{thm:non-robust}, we conclude that for $c^2 k=\omega_d(1)$, w.h.p. over a training dataset of size $n \geq k \ln^2(d)$, if gradient flow reaches empirical loss smaller than $\frac{1}{n}$, then it converges in direction to a neural network that generalizes well but is not $\sqrt{d}$-robust, even though there exist $\sqrt{d}$-robust networks that generalize well. Thus, in our setting, there is bias towards solutions that generalize well but are non-robust.

\ignore{
The above theorem implies that if $c^2 k=\omega_d(1)$, then w.h.p. over the training dataset, every KKT point of Problem~(\ref{eq:optimization problem}) is not $\sqrt{d}$-robust. Recall that by \thmref{thm:generalizaton}, we also have w.h.p. that every KKT point generalizes well. Overall, combining Theorems~\ref{thm:known KKT},~\ref{thm:generalizaton},~\ref{thm:robust exists}, and~\ref{thm:non-robust}, we conclude that for $c^2 k=\omega_d(1)$, w.h.p. over a training dataset of size $k \ln^2(d) \leq n \leq \poly(d)$, if gradient flow reaches empirical loss smaller than $\frac{1}{n}$, then it converges in direction to a neural network that generalizes well but is not $\sqrt{d}$-robust, even though there exist $\sqrt{d}$-robust networks that generalize well. Thus, in our setting, there is bias towards solutions that generalize well but are non-robust.
}

\begin{example}
    Consider the setting from the first item of \exampleref{ex:dist}. Thus, the cluster means satisfy $|\inner{\bmu^{(i)}, \bmu^{(j)}} | = \tilde{\co}(\sqrt{d})$ for every $i \neq j$, and we have $\sigma=1$ and $k = \tilde{\Theta}(\sqrt{d})$. Suppose that $c=\Theta(1)$, namely, there is at least a constant fraction of clusters with each label $\{-1,1\}$. Then, the adversarial perturbation $\bz$ from \thmref{thm:non-robust} satisfies 
    \[
        \norm{\bz} 
        = \co\left( \sqrt{\frac{d}{k}} \right)
        = \tilde{\co}\left( d^{1/4} \right)
        = o(\sqrt{d})~.
    \]
\end{example}

Similarly to our discussion after \thmref{thm:generalizaton}, we note that \thmref{thm:non-robust} is width-independent, i.e., it holds irrespective of the network width. It implies that we cannot hope to obtain a robust solution by choosing an appropriate width for the trained network.
As we discussed in the related work section, \citet{bubeck2021law} and \citet{bubeck2021universal} considered the expressive power of neural networks, and showed that overparameterization might be necessary for robustness. By \thmref{thm:non-robust}, even when the network is overparameterized, the implicit bias of the optimization method 
can 
prevent convergence to robust solutions.  

Moreover, our result does not depend directly on the initialization of gradient flow. Recall that by \thmref{thm:known KKT} if gradient flow reaches small empirical loss then it converges in direction to a KKT point of Problem~(\ref{eq:optimization problem}). Hence our result holds whenever gradient flow reaches a small empirical loss.

Note that in \thmref{thm:non-robust}, the adversarial perturbation does not depend on the input (up to sign). It corresponds to the well-known empirical phenomenon of \emph{universal adversarial perturbations}, where one can find a single perturbation that simultaneously flips the label of many inputs (cf. \cite{moosavi2017universal,zhang2021survey}).
Moreover, the same perturbation applies to all depth-$2$ networks to which gradient flow might converge (i.e., all KKT points). It corresponds to the well-known empirical phenomenon of \emph{transferability} in adversarial examples, where one can find perturbations that simultaneously flip the labels of many different trained networks (cf. \cite{liu2017delving,akhtar2018threat}).

It is worth noting that Theorems~\ref{thm:generalizaton} and~\ref{thm:non-robust} demonstrate that trained neural networks exhibit different properties than the $1$-nearest-neighbour learning rule, irrespective of the number of parameters in the network. For example, consider the case where $\sigma=\frac{1}{\sqrt{d}}$, namely, the examples of each cluster are concentrated within a ball of radius $O(1)$ around its mean. Then, the distance between every pair of points from the same cluster is $O(1)$, and the distance between points from different clusters is $\Omega(\sqrt{d})$. In this setting, 
both the $1$-nearest-neighbour classifier and the trained neural network will classify a fresh example correctly w.h.p., but in the $1$-nearest-neighbour classifier flipping the output's sign will require a perturbation of size $\Omega(\sqrt{d})$, while in the neural network a much smaller perturbation will suffice. 

Finally, we remark that in the limit $\sigma \to 0$, we get a distribution supported on $\bmu^{(1)},\ldots,\bmu^{(k)}$. Then, a training dataset of size $n \geq k \ln^2(d)$ will contain w.h.p. all examples in the support, and hence robustness w.r.t. test data is equivalent to robustness w.r.t. the training data.  In this case, we recover the results of~\citet{vardi2022gradient} which characterized the non-robustness of KKT points of ReLU networks trained on nearly orthogonal training data.  In particular, our \thmref{thm:non-robust} is a strict generalization of their Theorem~4.1.
\ignore{ 
Therefore, in this case \thmref{thm:non-robust} is similar to the non-robustness guarantee w.r.t. nearly-orthogonal training data that was established in \citet{vardi2022gradient}. Indeed, for the training examples $\{\bmu^{(1)},\ldots,\bmu^{(k)}\}$, Theorem~4.1 in \cite{vardi2022gradient} guarantees an adversarial perturbation of size $\co\left( \sqrt{\frac{d}{c^2 k}} \right)$ for each example in the dataset, which is exactly the bound from \thmref{thm:non-robust}. 
}

\subsection{Proof ideas}

Here we discuss the main ideas in the proofs of Theorems~\ref{thm:robust exists} and~\ref{thm:non-robust}. See Appendices~\ref{app:proof of robust exists} and~\ref{app:proof of non-robust} for the formal proofs.

\subsubsection{\thmref{thm:robust exists}}

The proof follows by the following simple construction.
The robust network includes $k$ neurons, each corresponding to a single cluster. 
That is, we have $\cn(\bx) = \sum_{j=1}^k v_j \sigma(\bw^\top \bx + b_j)$, where $v_j=y^{(j)}$, $\bw_j = \frac{4\mu^{(j)}}{d}$, and $b_j=-2$.
Note that the $j$-th neuron points at the direction of the $j$-th cluster and has a negative bias term, such that the neuron is active on points from the $j$-th cluster, and inactive on points from the other clusters. 
Then, given a fresh example $(\bx,y) \sim \Dclust$, we show that the network classifies it correctly w.h.p. with margin at least $1$. 
Also, there is w.h.p exactly one neuron that is active on $\bx$, and hence the gradient of the network w.r.t. the input is affected only by this neuron and is of size $\co(1/\sqrt{d})$. Therefore, we need a perturbation of size $\Omega(\sqrt{d})$ in order to flip the output's sign.    
\ignore{
We remark that a similar construction was used in \cite{vardi2022gradient} to show the existence of networks that are robust w.r.t. the training data for nearly-orthogonal datasets.
}

\subsubsection{\thmref{thm:non-robust}}

While the proof of \thmref{thm:robust exists} follows by a simple construction, the proof of \thmref{thm:non-robust} is more challenging.
The intuition for this result can be described as follows. Recall that in our construction of a robust network in the previous subsection, an example $(\bx,y)\sim \Dclust$ is w.h.p. in an active region of exactly one neuron, and hence in the neighborhood of $\bx$ the output of the network is sensitive only to perturbations in the direction of that neuron. Now, consider the linear model $\bx \mapsto \bw^\top \bx$, where $\bw = \sum_{q=1}^k \frac{1}{d} y^{(q)} \bmu^{(q)}$. It is not hard to verify that for $(\bx,y) \sim \Dclust$ we have w.h.p. that $0 < y \bw^\top \bx \leq \co(1)$. Moreover, the gradient of this linear predictor is of size $\norm{\bw} = \Omega(\sqrt{k/d})$. Hence, we can flip the output's sign with a perturbation of size $\co(\sqrt{d/k})$. Thus, the linear classifier is non-robust if $k=\omega_d(1)$. Intuitively, the difference between our robust ReLU network and the non-robust linear classifier is the fact that in the neighborhood of $\bx$ the robust network is sensitive only to perturbations in the direction of one cluster, while the linear classifier is sensitive to perturbations in the directions of all $k$ clusters. In the proof, we analyze ReLU networks which are KKT points of Problem~(\ref{eq:optimization problem}), and show that although these ReLU networks are non-linear, they are still sensitive to perturbations in the directions of all $k$ clusters, similarly to the aforementioned linear classifier. 

The formal proof follows by a careful analysis of the KKT conditions of Problem~(\ref{eq:optimization problem}),
given in Eq.~(\ref{eq:kkt condition w idea}) and~(\ref{eq:kkt condition b idea}).
We show that for every solution $\btheta$ of these equations, we have w.h.p. over $(\bx,y) \sim \Dclust$ that $|\cn_\btheta(\bx)| = \co(1)$, and that a perturbation in the direction of $\bz$ and size $s$ changes the output by $\Omega\left(s \cdot \sqrt{\frac{c^2 k}{d}} \right)$. 
Recall that in \subsecref{subsec:generalization ideas}, we mentioned that Eq.~(\ref{eq:kkt condition w idea}) and~(\ref{eq:kkt condition b idea}) imply that the positive (respectively, negative) neurons have ``large'' components in the directions of all positive (respectively, negative) clusters. In the proof, we use this property in order to show that perturbations in the direction of $\bz$, which has components in the directions of all $k$ clusters, change the network's output sufficiently fast.

We remark that in the proof of \thmref{thm:non-robust} we use some technical ideas from \citet{vardi2022gradient}. However, there are significant differences between the two settings. For example, they assume that the training data are nearly orthogonal, which only holds when the dimension is large relative to the number of samples; thus, it is unclear whether the existence of small adversarial perturbations in their setting is due to the high-dimensionality of the data or if a similar phenomenon exists in the more common $n> d$ setting.  At a more technical level, their proof relies on showing that in a KKT point all inputs must lie exactly on the margin, while in our setting they are not guaranteed to lie exactly on the margin.

\section{Discussion} \label{sec:discussion}

In this paper, we considered clustered data, and showed that gradient flow in two-layer ReLU networks does not 
harmfully 
overfit, but also 
hinders 
robustness. 
Our results follow by analyzing the KKT points of the max-margin problem in parameter space.
In our distributional setting, the clusters are well-separated, and hence there exist robust classifiers, which allows us to consider the effect of the implicit bias of gradient flow on both generalization and robustness.
Understanding generalization and robustness in 
additional
data distributions and neural network architectures is a challenging but important question. As a possible next step, it would be interesting to study whether the approach used in this paper can be extended to the following data distributions:

First, our assumption on the data distribution (\assref{ass:dist}) implies that the number of clusters cannot be too large, and as a result the data is linearly separable (\lemref{lem:linearly separable}). We conjecture that our results hold even for a significantly larger number of clusters, such that the data is not linearly separable. 

Second, it would be interesting to understand whether our generalization result 
holds for linearly separable data distributions that are not clustered. That is, given a distribution that is linearly separable with some margin $\gamma>0$ and a training dataset that is large enough to allow learning with a max-margin linear classifier, are there KKT points of the max-margin problem for two-layer ReLU networks that do not generalize well? In other words, do ReLU networks that satisfy the KKT conditions generalize at least as well as max-margin linear classifiers? 

\subsection*{Acknowledgements}

SF, GV, PB, and NS acknowledge the support of the NSF and the Simons Foundation for the Collaboration on the Theoretical Foundations of Deep Learning through awards DMS-2031883 and \#814639, and of the NSF
through grant DMS-2023505.

\printbibliography

\newpage

\appendix


\section{Proof of \thmref{thm:generalizaton}} \label{app:generalization}


We will prove the following theorem:
\begin{theorem} \label{thm:generalizaton orig}
	Let $\cs = \{(\bx_i,y_i)\}_{i=1}^n \subseteq \reals^d \times \{-1,1\}$ be a training set drawn i.i.d. from the distribution $\Dclust$, where $n \geq k \ln^2(d)$. Let $\cn_\btheta$ be a depth-$2$ ReLU network such that $\btheta = [\bw_1,\ldots,\bw_m,\bb,\bv]$ is a KKT point of Problem~(\ref{eq:optimization problem}). 
	Then, with probability at least 
	\[
		1 - \left(3 n^2 d^{-\frac{\ln(d)}{2}} + n^2 e^{-d/16} + d^{1-\ln(d)} \right) 
	\]	
	over $\cs$, 
	we have 
	\[
		\Pr_{(\bx,y) \sim \Dclust} \left[ y \cn_\btheta(\bx)) \leq 0 \right] 
		\leq 4 n d^{-\frac{\ln(d)}{2}}~.
	\]
\end{theorem}

It is easy to verify that \thmref{thm:generalizaton orig} implies \thmref{thm:generalizaton}. Indeed, if 
$\frac{1}{\delta} \leq \frac{1}{3} d^{\ln(d) - 1}$
and 
\[
    n \leq \min\left\{ \sqrt{\frac{\delta}{3}} \cdot e^{d/32}, \frac{\sqrt{\delta}}{3} \cdot d^{\ln(d)/4}, \frac{\epsilon}{4} \cdot d^{\ln(d)/2} \right\}~,
\]    
then we have:
\begin{enumerate}
    \item 
        \begin{align*}
            3 n^2 d^{-\frac{\ln(d)}{2}}
            \leq 3 \left(\frac{\sqrt{\delta}}{3} d^{\ln(d)/4}\right)^2 d^{-\ln(d)/2}
            = 3 \cdot \frac{\delta}{9} \cdot d^{\ln(d)/2} d^{-\ln(d)/2}
            = \frac{\delta}{3}~.
        \end{align*}
    \item 
        \[
            n^2 e^{-d/16}
            \leq \left( \sqrt{\frac{\delta}{3}} \cdot e^{d/32} \right)^2 e^{-d/16}
            = \frac{\delta}{3}~.
        \]
    \item 
        \[  
            d^{1-\ln(d)} \leq \frac{\delta}{3}~.   
        \]
    \ignore{
        \begin{align*}
            d^{1-\ln(d)}
            &= \exp\left(\ln(d)(1-\ln(d))\right)
            \leq \exp\left(1-\ln(d)\right)
            \leq \exp\left(1-\ln(9/\delta)\right)
            \\
            &\leq \exp\left(1-\ln(3e/\delta)\right)
            = \exp\left(1-(1+\ln(3/\delta))\right)
            = \exp\left(-\ln(3/\delta))\right)
            = \frac{\delta}{3}~.
        \end{align*}
    }
    \item 
        \[
            4 n d^{-\frac{\ln(d)}{2}}
            \leq 4 \cdot \frac{\epsilon}{4} \cdot d^{\ln(d)/2} \cdot d^{-\ln(d)/2}
            = \epsilon~. 
        \]
\end{enumerate}
Hence, under the above assumptions on $\delta$ and $n$,  \thmref{thm:generalizaton orig} implies that with probability at least $1-\delta$ over $\cs$, we have $\Pr_{(\bx,y) \sim \Dclust} \left[ y \cn_\btheta(\bx)) \leq 0 \right] \leq \epsilon$.

We now turn to prove \thmref{thm:generalizaton orig}. The high-level idea for the proof is as follows.  First, we show that the training dataset is sufficiently ``nice'' with high probability, in the sense that samples within each cluster are highly correlated while samples in orthogonal clusters are nearly orthogonal (see properties~\ref{p:noise.norm} through~\ref{p:sample.in.every.cluster} below).  This analysis appears in Section~\ref{app:generalization.nicedata}.   We then show that datasets with ``nice'' properties impose a number of structural constraints on the properties of KKT points of the margin maximization problem for ReLU nets; this appears in Section~\ref{app:generalization.kkt.implications}.  We conclude in Section~\ref{app:generalization.conclusion} by showing how these structural conditions allow for generalization on fresh test data.

\subsection{Training dataset properties}\label{app:generalization.nicedata}

We denote $\cn_\btheta(\bx) = \sum_{j \in [m]} v_j \phi(\bw_j^\top \bx + b_j)$. Thus, $\cn_\btheta$ is a network of width $m$, where the weights in the first layer are $\bw_1,\ldots,\bw_m$, the bias terms are $b_1,\ldots,b_m$, and the weights in the second layer are $v_1,\ldots,v_m$. 
We denote $J := [m]$, $J_+ := \{j \in J: v_j > 0\}$, and $J_- := \{j \in J: v_j < 0\}$. 
Moreover, we denote $I := [n]$, $I_+ := \{i \in I: y_i =1\}$, and $I_- := \{i \in I: y_i = -1\}$. Finally, we denote $Q := [k]$, $Q_+ := \{q \in Q: y^{(q)} = 1\}$, and  $Q_- := \{q \in Q: y^{(q)} = -1\}$.

We denote $p := \max_{q \neq q'} | \inner{\bmu^{(q)},\bmu^{(q')}} |$. 
The distribution $\Dclust$ is such that each example $(\bx_i,y_i)$ in $\cs$ is generated as follows: we draw $q_i \sim \cu(Q)$ and $\bxi_i \sim \ndist(\zero, \sigma^2 I_d)$ and set $\bx_i  = \bmu^{(q_i)} + \bxi_i$ and $y_i=y^{(q_i)}$. We denote $\cluster(i) = q_i$. For $q \in Q$ we denote $I^{(q)} = \{i \in I: \cluster(i)=q\}$. We also denote $\Delta = 4\sigma \sqrt{d} \ln(d)$.

Our goal in this section will be to show that with high probability, the dataset $\cs$ satisfies the following properties.  
\begin{enumerate}[label=(P\arabic*)]
    \item \label{p:noise.norm} For every $i \in I$ we have $\norm{\bxi_i} \leq \sigma \sqrt{2d}$.
    \item For every $i \neq i'$ in $I$ we have $|\inner{\bxi_i,\bxi_{i'}}| \leq \sigma^2 \sqrt{2d} \ln(d)$.
    \item For every $i \in I$ and $q \in Q$ we have $| \inner{\bmu^{(q)},\bxi_i} | \leq \sigma \sqrt{d} \ln(d)$.
	\item For every $i,i' \in I$ with $\cluster(i) \neq \cluster(i')$ we have $ | \inner{\bx_i,\bx_{i'}} | \leq p + \Delta$.
	\item For every $i,i' \in I$ with $\cluster(i) = \cluster(i')$ we have $d - \Delta \leq \inner{\bx_i,\bx_{i'}} \leq 3d + \Delta$.
	\item \label{p:sample.in.every.cluster}For every $q \in Q$ there exists $i \in I$ with $\cluster(i)=q$ (i.e., $I^{(q)} \neq \emptyset$).
\end{enumerate}
More formally, in the remainder of this section we shall show the following proposition.

\begin{proposition}\label{prop:trainingdata}
With probability at least $1-\left(3n^2 d^{-\frac{\ln(d)}{2}} + n^2 e^{-d/16} + d^{1-\ln(d)}\right)$, the dataset $\cs$ satisfies the properties~\ref{p:noise.norm} through~\ref{p:sample.in.every.cluster}. 
\end{proposition}

We start with some auxiliary lemmas.  The first bounds the norm of $\bxi$. 

\begin{lemma}  \label{lem:bound xi norm}
	Let $\bxi \sim \ndist(\zero,\sigma^2 I_d)$. Then, 
	\[
		\Pr\left[ \norm{\bxi}  \geq \sigma \sqrt{2d} \right] 
		\leq e^{-d/16}~.
	\]
\end{lemma}
\begin{proof}
	Note that $\norm{\frac{\bxi}{\sigma}}^2$ has the Chi-squared distribution. A concentration bound by Laurent and Massart \citep[Lemma~1]{laurent2000adaptive} implies that for all $t > 0$ we have
	\[
	 	\Pr\left[ \norm{\frac{\bxi}{\sigma}}^2 - d \geq 2\sqrt{dt} + 2t \right] \leq e^{-t}~.
	\]
	Plugging-in $t=\frac{d}{16}$, we get
	\begin{align*}
        \Pr\left[ \norm{\frac{\bxi}{\sigma}}^2  \geq  2d \right] 
		\leq \Pr\left[ \norm{\frac{\bxi}{\sigma}}^2 - d \geq  d/2 + d/8 \right] 
		\leq e^{-d/16}~.
	\end{align*}
\ignore{	
	and therefore
	\begin{align*}
		\Pr\left[ \norm{\frac{\bxi}{\sigma}}^2  \geq  2d \right] 
		\leq e^{-d/16}~.
	\end{align*}
}
	Thus, we have
	\begin{equation*} 
		\Pr\left[ \norm{\bxi}  \geq \sigma \sqrt{2d} \right] 
		\leq e^{-d/16}~.
	\end{equation*}
\end{proof}

Our next lemma bounds the projection of a Gaussian $\bxi'$ onto a fixed vector $\bxi$. 

\begin{lemma} \label{lem:bound xis product fixed}
	Let $\bxi \in \reals^d$ and let $\bxi' \sim \ndist(\zero,\sigma^2 I_d)$. Then, 
	\[
		\Pr\left[ | \inner{\bxi, \bxi'} | \geq \norm{\bxi} \sigma \ln(d) \right] 
		\leq 2 d^{-\frac{\ln(d)}{2}}~.
	\]
\end{lemma}
\begin{proof}
	Note that $\inner{\frac{\bxi}{\norm{\bxi}}, \bxi'}$ has the distribution $\ndist(0,\sigma^2)$. By a standard tail bound, we have for every $t \geq 0$ that $\Pr\left[ \left| \inner{\frac{\bxi}{\norm{\bxi}}, \bxi'} \right| \geq t \right] \leq 2\exp\left(-\frac{t^2}{2\sigma^2}\right)$. Hence,
	\begin{align*}
		\Pr\left[ \left| \binner{\frac{\bxi}{\norm{\bxi}}, \bxi'} \right| \geq  \sigma \ln(d) \right] 
		\leq 2 \exp\left(-\frac{\sigma^2 \ln^2(d)}{2\sigma^2}\right)
		= 2 d^{-\frac{\ln(d)}{2}}~.
	\end{align*}
	The lemma now follows immediately.
\end{proof}

We can utilize the two preceding lemmas to show that the pairwise correlations between independent Gaussians is small relative to the norms of the Gaussians. 

\begin{lemma} \label{lem:bound xis product}
	Let $\bxi,\bxi'$ drawn i.i.d. from $\ndist(\zero,\sigma^2 I_d)$. Then, 
	\[
		\Pr\left[ |\inner{\bxi,\bxi'}| \geq \sqrt{2d} \ln(d) \sigma^2 \right]
		\leq e^{-d/16} + 2d^{-\frac{\ln(d)}{2}}~.
	\]
\end{lemma}
\begin{proof}
	Note that if $|\inner{\bxi,\bxi'}| \geq \sqrt{2d} \ln(d) \sigma^2$ then we have at least one of the following: (1) $\norm{\bxi} \geq  \sigma \sqrt{2d}$; (2) $|\inner{\frac{\bxi}{\norm{\bxi}}, \bxi'}| \geq \sigma \ln(d)$. We will bound the probability of these events.

	First, by \lemref{lem:bound xi norm} we have 
	\begin{equation*} \label{eq:chi bound}
		\Pr\left[ \norm{\bxi}  \geq \sigma \sqrt{2d} \right] 
		\leq e^{-d/16}~.
	\end{equation*}
	Next, 
	by \lemref{lem:bound xis product fixed} we have 
	\begin{align*}
		\Pr\left[ \left| \binner{\frac{\bxi}{\norm{\bxi}}, \bxi'} \right| \geq  \sigma \ln(d) \right] 
		\leq 2d^{-\frac{\ln(d)}{2}}~.
	\end{align*}
	Combining the above 
	displayed equations
	we conclude that 
	\[
		\Pr\left[ |\inner{\bxi,\bxi'}| \geq \sqrt{2d} \ln(d) \sigma^2 \right] 
		\leq e^{-d/16} + 2d^{-\frac{\ln(d)}{2}}~.
	\]
\end{proof}

Our next lemma bounds the projection of the noise vectors onto any cluster mean.  

\begin{lemma} \label{lem:bound mu xi product}
	Let $i \in [k]$ and let $\bxi \sim \ndist(\zero,\sigma^2 I_d)$. Then
	\[
		\Pr \left[ |\inner{\bmu^{(i)}, \bxi} | \geq \sigma \sqrt{d} \ln(d) \right] 
		\leq 2d^{-\frac{\ln(d)}{2}}~.
	\]
\end{lemma}
\begin{proof}
	Follows immediately from \lemref{lem:bound xis product fixed}.
\ignore{
	Note that $\inner{\frac{\bmu^{(i)}}{\norm{\bmu^{(i)}}}, \bxi}$ has the distribution $\ndist(0,\sigma^2)$. By a standard tail bound, we have for every $t \geq 0$ that $\Pr\left[ \inner{\frac{\bmu^{(i)}}{\norm{\bmu^{(i)}}}, \bxi'} \geq t \right] \leq \exp\left(-\frac{t^2}{2\sigma^2}\right)$. Hence,
	\begin{align*}
		\Pr\left[ | \inner{\bmu^{(i)}, \bxi} | \geq  \sigma \sqrt{d} \ln(d) \right] 
		= \Pr\left[ \left| \binner{\frac{\bmu^{(i)}}{\norm{\bmu^{(i)}}}, \bxi} \right| \geq  \sigma \ln(d) \right] 
		\leq \exp\left(-\frac{\sigma^2 \ln^2(d)}{2\sigma^2}\right)
		= d^{-\frac{\ln(d)}{2}}~.
	\end{align*}
}
\end{proof}

The following lemmas use bounds on the pairwise interactions between noise vectors and cluster means to bound the correlations between the sums of noises and cluster means.   

\begin{lemma} \label{lem:bound points product ij}
    Let $i \neq j$ be indices in $[k]$. Let $\bxi,\bxi'$ such that the following hold:
    \begin{itemize}
        \item $| \inner{\bmu^{(i)}, \bxi'} | \leq  \sigma \sqrt{d} \ln(d)$.
        \item $| \inner{\bmu^{(j)}, \bxi} | \leq  \sigma \sqrt{d} \ln(d)$.
        \item $|\inner{\bxi,\bxi'}| \leq \sigma \sqrt{2d} \ln(d) $.
    \end{itemize}
    Then, 
    \[
        |\inner{\bmu^{(i)} + \bxi, \bmu^{(j)} + \bxi' } | 
        \leq 4\sigma \sqrt{d} \ln(d) + | \inner{\bmu^{(i)},\bmu^{(j)}} |~.
    \]
\end{lemma}
\begin{proof}
    We have
    \begin{align*}
        |\inner{\bmu^{(i)} + \bxi, \bmu^{(j)} + \bxi' } | 
        &\leq | \inner{\bmu^{(i)},\bmu^{(j)}} | + | \inner{\bmu^{(i)},\bxi'} | + | \inner{\bxi,\bmu^{(j)}} | + | \inner{\bxi,\bxi'} |
        \\
        &\leq | \inner{\bmu^{(i)},\bmu^{(j)}} | + \sigma \sqrt{d} \ln(d) + \sigma \sqrt{d} \ln(d) + \sigma \sqrt{2d} \ln(d) 
        \\
        &\leq | \inner{\bmu^{(i)},\bmu^{(j)}} | + 4 \sigma \sqrt{d} \ln(d)~.
    \end{align*}
\end{proof}

\ignore{
\begin{lemma} \label{lem:bound points product ij}
	Let 
	$i \neq j$ be indices in $[k]$ 
	and let $\bxi,\bxi'$ drawn i.i.d. from $\ndist(\zero,\sigma^2 I_d)$. Then
	\[
		\Pr \left[ |\inner{\bmu^{(i)} + \bxi, \bmu^{(j)} + \bxi' } | \geq 4\sigma \sqrt{d} \ln(d) + | \inner{\bmu^{(i)},\bmu^{(j)}} | \right] 
		\leq 6 d^{-\frac{\ln(d)}{2}} + e^{-d/16}~.
	\]
\end{lemma}
\begin{proof}
	By \lemref{lem:bound mu xi product} we have 
	\begin{align*}
		\Pr\left[ | \inner{\bmu^{(i)}, \bxi'} | \geq  \sigma \sqrt{d} \ln(d) \right] 
		\leq 2 d^{-\frac{\ln(d)}{2}}~.
	\end{align*}
	A similar bound holds for $|\inner{\bmu^{(j)},\bxi}|$. 
	By \lemref{lem:bound xis product} we have 
	\[
		\Pr\left[ |\inner{\bxi,\bxi'}| \geq \sqrt{2d} \ln(d) \sigma^2 \right]
		\leq e^{-d/16} + 2d^{-\frac{\ln(d)}{2}}~.
	\]
	
	Combining it with $\sigma \leq 1$ we get
	\begin{align*}
		\Pr &\left[ |\inner{\bmu^{(i)} + \bxi, \bmu^{(j)} + \bxi' } | \geq 4\sigma \sqrt{d} \ln(d) + | \inner{\bmu^{(i)},\bmu^{(j)}} | \right] 
		\\
		&\leq \Pr \left[ |\inner{\bmu^{(i)} + \bxi, \bmu^{(j)} + \bxi' } | \geq 2\sigma \sqrt{d} \ln(d) +  \sqrt{2d} \ln(d) \sigma^2 + | \inner{\bmu^{(i)},\bmu^{(j)}} | \right] 
		\\
		&\leq \Pr \left[ |\inner{\bmu^{(i)}, \bxi'}| \geq \sigma \sqrt{d} \ln(d)  \text{ or } |\inner{\bxi,\bmu^{(j)}}| \geq \sigma \sqrt{d} \ln(d)  \text{ or } |\inner{\bxi,\bxi'}| \geq   \sqrt{2d} \ln(d) \sigma^2 \right] 
		\\
		&\leq 4 d^{-\frac{\ln(d)}{2}} + e^{-d/16} + 2d^{-\frac{\ln(d)}{2}}
		= 6 d^{-\frac{\ln(d)}{2}} + e^{-d/16}~. 
	\end{align*}
\end{proof}
}

\begin{lemma} \label{lem:bound points product ii}
    Let $i \in [k]$, and let $\bxi,\bxi'$ such that the following hold:
    \begin{itemize}
        \item $| \inner{\bmu^{(i)}, \bxi} | \leq  \sigma \sqrt{d} \ln(d)$.
        \item $| \inner{\bmu^{(i)}, \bxi'} | \leq \sigma \sqrt{d} \ln(d)$.
        \item $|\inner{\bxi,\bxi'}| \leq \sigma \sqrt{2d} \ln(d)$.
    \end{itemize}
    Then, 
    \[
        \left| \inner{\bmu^{(i)} + \bxi, \bmu^{(i)} + \bxi' } - d \right| 
        \leq  4\sigma \sqrt{d} \ln(d)~.
    \]
\end{lemma}
\begin{proof}
    We have
    \begin{align*}
        \left| \inner{\bmu^{(i)} + \bxi, \bmu^{(i)} + \bxi' } - d \right| 
        &= \left| \inner{\bmu^{(i)},\bxi'} + \inner{\bxi, \bmu^{(i)}} + \inner{\bxi,\bxi'} \right| 
        \\
        &\leq \sigma \sqrt{d} \ln(d) + \sigma \sqrt{d} \ln(d) + \sigma \sqrt{2d} \ln(d)
        \\
        &\leq 4\sigma \sqrt{d} \ln(d)~.
    \end{align*}
    
\end{proof}

\ignore{
\begin{lemma} \label{lem:bound points product ii}
	Let $i \in [k]$ and let $\bxi,\bxi'$ drawn i.i.d. from $\ndist(\zero,\sigma^2 I_d)$. Then
	\[
		\Pr \left[ \left| \inner{\bmu^{(i)} + \bxi, \bmu^{(i)} + \bxi' } - d \right| \geq  4\sigma \sqrt{d} \ln(d) \right] 
		\leq 6 d^{-\frac{\ln(d)}{2}} + e^{-d/16}~.
	\]
\end{lemma}
\begin{proof}
	By \lemref{lem:bound mu xi product} we have 
	\begin{align*}
		\Pr\left[ | \inner{\bmu^{(i)}, \bxi'} | \geq  \sigma \sqrt{d} \ln(d) \right] 
		\leq 2d^{-\frac{\ln(d)}{2}}~,
	\end{align*}
	and similarly for $ | \inner{\bxi, \bmu^{(i)}} |$. Combining it with \lemref{lem:bound xis product}, with $\inner{\bmu^{(i)},\bmu^{(i)}} = d$, and with $\sigma \leq 1$, we get
	\begin{align*}
		\Pr & \left[ \left| \inner{\bmu^{(i)} + \bxi, \bmu^{(i)} + \bxi' } - d \right| \geq  4\sigma \sqrt{d} \ln(d) \right] 
		\\
		&\leq \Pr \left[ \left| \inner{\bmu^{(i)} + \bxi, \bmu^{(i)} + \bxi' } - d \right| \geq  2\sigma \sqrt{d} \ln(d) + \sqrt{2d} \ln(d) \sigma^2 \right] 
		\\
		&= \Pr\left[ \left| \inner{\bmu^{(i)}, \bxi' } + \inner{\bxi, \bmu^{(i)}} + \inner{\bxi, \bxi' } \right|  \geq  2\sigma \sqrt{d} \ln(d) + \sqrt{2d} \ln(d) \sigma^2 \right]  
		\\
		&\leq \Pr \left[ |\inner{\bmu^{(i)}, \bxi'}| \geq \sigma \sqrt{d} \ln(d)  \text{ or } |\inner{\bxi,\bmu^{(i)}}| \geq \sigma \sqrt{d} \ln(d)  \text{ or } |\inner{\bxi,\bxi'}| \geq   \sqrt{2d} \ln(d) \sigma^2 \right] 
		\\
		&\leq 4 d^{-\frac{\ln(d)}{2}} + e^{-d/16} + 2d^{-\frac{\ln(d)}{2}}
		= 6 d^{-\frac{\ln(d)}{2}} + e^{-d/16}~. 
	\end{align*}
\end{proof}
}

The next lemma uses bounds on the projection of the noise vector onto cluster means and the norms of the cluster means to derive bounds on the norm of the sum $\bmu^{(i)}+\xi$.

\begin{lemma} \label{lem:bound point norm}
    Let $i \in [k]$, and let $\bxi$ such that the following hold:
    \begin{itemize}
        \item $| \inner{\bmu^{(i)}, \bxi} | \leq  \sigma \sqrt{d} \ln(d)$.
        \item $\norm{\bxi}^2 \leq 2 \sigma^2 d$.
    \end{itemize}
    Then, 
    \[
        d - 2\sigma \sqrt{d} \ln(d) \leq \norm{\bmu^{(i)} + \bxi}^2 \leq 3d + 2 \sigma \sqrt{d} \ln(d)~. 
    \]
\end{lemma}
\begin{proof}
    We have 
    \begin{align*}
	    \norm{\bmu^{(i)} + \bxi}^2
	    &= \norm{\bmu^{(i)}}^2 + \norm{\bxi}^2 + 2 \inner{\bmu^{(i)},\bxi}
	    \geq \norm{\bmu^{(i)}}^2 - 2 \left| \inner{\bmu^{(i)},\bxi} \right|
	    \geq d - 2 \sigma \sqrt{d} \ln(d)~,
	\end{align*}
	and
	\begin{align*}
	    \norm{\bmu^{(i)} + \bxi}^2
	    &= \norm{\bmu^{(i)}}^2 + \norm{\bxi}^2 + 2 \inner{\bmu^{(i)},\bxi}
	    \leq d + 2 \sigma^2 d + 2 \sigma \sqrt{d} \ln(d) 
	    \leq 3d + 2 \sigma \sqrt{d} \ln(d)~. 
	\end{align*}
\end{proof}

\ignore{
\begin{lemma} \label{lem:bound point norm}
	Let $i \in [k]$ and let $\bxi \sim \ndist(\zero,\sigma^2 I_d)$. Then
	\[
		\Pr \left[ d - 2\sigma \sqrt{d} \ln(d) \leq \norm{\bmu^{(i)} + \bxi}^2 \leq 3d + 2 \sigma \sqrt{d} \ln(d) \right] 
		\geq 1 - \left( 2d^{-\frac{\ln(d)}{2}} + e^{-d/16} \right)~.
	\]
\end{lemma}
\begin{proof}
    By \lemref{lem:bound mu xi product} we have 
    \begin{align*}
		\Pr\left[ | \inner{\bmu^{(i)}, \bxi} | \geq \sigma \sqrt{d} \ln(d) \right] 
		\leq 2d^{-\frac{\ln(d)}{2}}~.
	\end{align*}
	By \lemref{lem:bound xi norm}, we have
	\[
	    \Pr \left[ \norm{\bxi}^2 \geq 2 \sigma^2 d \right] 
	    \leq e^{-d/16}~.
	\]
	Hence, with probability at least $1 - \left( 2d^{-\frac{\ln(d)}{2}} + e^{-d/16} \right)$ we have
	\begin{align*}
	    \norm{\bmu^{(i)} + \bxi}^2
	    &= \norm{\bmu^{(i)}}^2 + \norm{\bxi}^2 + 2 \inner{\bmu^{(i)},\bxi}
	    \geq \norm{\bmu^{(i)}}^2 - 2 \left| \inner{\bmu^{(i)},\bxi} \right|
	    > d - 2 \sigma \sqrt{d} \ln(d)~,
	\end{align*}
	and
	\begin{align*}
	    \norm{\bmu^{(i)} + \bxi}^2
	    &= \norm{\bmu^{(i)}}^2 + \norm{\bxi}^2 + 2 \inner{\bmu^{(i)},\bxi}
	    < d + 2 \sigma^2 d + 2 \sigma \sqrt{d} \ln(d) 
	    \leq 3d + 2 \sigma \sqrt{d} \ln(d)~. 
	\end{align*}
\end{proof}
}

Our final lemma in this section shows that each cluster contains some examples with high probability.  
\begin{lemma} \label{lem:bound coupons}
	With probability at least $1-d^{1-\ln(d)}$ the dataset $\cs$ contains at least one example from each cluster in $[k]$.
\end{lemma}
\begin{proof}
	Note that this problem corresponds to the ``coupons collector's problem''. 
	The probability that $\cs$ does not contain points from cluster $j$ is at most 
	\[
		\left( 1 - \frac{1}{k} \right)^n 
		\leq \exp \left(- \frac{n}{k} \right)
		\leq \exp \left(- \ln^2(d) \right)
		= d^{-\ln(d)}~,
	\]
	where in the second inequality we used $n \geq k \ln^2(d)$.
	By the union bound, the probability that there is a cluster that does not appear in $\cs$ is at most $k \cdot d^{-\ln(d)}$. Since $k \leq d$ this probability is at most $d^{1-\ln(d)}$.
\end{proof}

The proof of Proposition~\ref{prop:trainingdata} now follows by putting together Lemmas~\ref{lem:bound xi norm},~\ref{lem:bound xis product},~\ref{lem:bound mu xi product},~\ref{lem:bound points product ij},~\ref{lem:bound points product ii},~\ref{lem:bound point norm},and ~\ref{lem:bound coupons}, and using $\sigma \leq 1$.

\subsection{Structural implications of the KKT conditions}\label{app:generalization.kkt.implications}
In this section we show that if the dataset $\cs$ satisfies Properties~\ref{p:noise.norm} through~\ref{p:sample.in.every.cluster}, then the KKT conditions impose a number of constraints on the behavior of the neural network.  We shall show that these constraints imply that the network will generalize well to unseen test data.  The reader may find it useful to refer back to the beginning of Section~\ref{app:generalization.nicedata} before proceeding.

We first outline what types of structural conditions on the KKT points would be useful for understanding generalization.  Suppose that $\bx$ is a test example coming from cluster $r\in Q_+$.  Our hope is that $ \cn_\btheta(\bx)>0$ for such an example.  Recall that since $\btheta$ satisfies the KKT conditions of Problem~(\ref{eq:optimization problem}), then there are $\lambda_1,\ldots,\lambda_n$ such that for every $j \in J$ we have
\begin{equation}
\label{eq:kkt condition w}
	\bw_j = \sum_{i \in I} \lambda_i \nabla_{\bw_j} \left( y_i \cn_{\btheta}(\bx_i) \right) =  \sum_{i \in I} \lambda_i y_i v_j \phi'_{i,j} \bx_i~,
\end{equation}
where $\phi'_{i,j}$ is a subgradient of $\phi$ at $\bw_j^\top \bx_i + b_j$, i.e., if $\bw_j^\top \bx_i + b_j \neq 0$ then $\phi'_{i,j} = \onefunc[\bw_j^\top \bx_i + b_j \geq 0]$, and otherwise $\phi'_{i,j}$ is some value in $[0,1]$. Also we have $\lambda_i \geq 0$ for all $i$, and $\lambda_i=0$ if 
$y_i  \cn_{\btheta}(\bx_i) 
\neq 1$. Likewise, we have
\stam{
\begin{equation}
\label{eq:kkt condition v}
	v_j = \sum_{i \in I} \lambda_i \nabla_{v_j} \left( y_i \cn_{\btheta}(\bx_i) \right) =  \sum_{i \in I} \lambda_i y_i \phi(\bw_j^\top \bx_i + b_j)~.
\end{equation}
}
\begin{equation}
\label{eq:kkt condition b}
	b_j = \sum_{i \in I} \lambda_i \nabla_{b_j} \left( y_i \cn_{\btheta}(\bx_i) \right) =  \sum_{i \in I} \lambda_i y_i v_j \phi'_{i,j}~.
\end{equation}
Now, consider the network output for an input $\bx$,
\begin{equation} \label{eq:network.output}
    \cn_\btheta(\bx) =
    \sum_{j\in J_+} v_j \phi(\bw_j^\top \bx + b_j) + \sum_{j\in J_-} v_j \phi(\bw_j^\top \bx + b_j) \geq \sum_{j\in J_+} v_j (\bw_j^\top \bx + b_j) + \sum_{j\in J_-} v_j \phi(\bw_j^\top \bx + b_j), 
\end{equation}
where we have used $\phi(z) \geq z$ for all $z\in \mathbb R$.  This suggests the following possibility: if we can ensure that $\sum_{j\in J_+} v_j (\bw_j^\top \bx + b_j)$ is large and positive while $\sum_{j\in J_-} v_j \phi(\bw_j^\top \bx+b_j)$ is not too negative, then the network will accurately classify the example $\bx$.   Using the KKT conditions and that $r\in Q_+$ (so $y_i=1$ for $i\in I^{(r)}$),  the first term in the above decomposition is equal to
\ignore{
\begin{align*}
    \sum_{j\in J_+} v_j (\bw_j^\top \bx + b_j) &= \sum_{j \in J_+} v_j \left[ \sum_{i\in I}  \lambda_i y_i v_j\phi'_{i,j}(\bx_i^\top \bx+1) \right]\\
    & =\sum_{j \in J_+} \left[ \left( \sum_{i \in I^{(r)}} \lambda_i v_j^2 \phi'_{i,j} (\bx_i^\top \bx + 1) \right) + \sum_{i \in I\setminus I^{(r)}} \lambda_i y_i v_j^2 \phi'_{i,j} (\bx_i^\top \bx + 1) \right]~.
\end{align*}
}
\begin{align*}
    \sum_{j\in J_+} v_j (\bw_j^\top \bx + b_j) &= \sum_{j \in J_+} v_j \left[ \sum_{i\in I}  \lambda_i y_i v_j\phi'_{i,j}(\bx_i^\top \bx+1) \right]
    \\
    &= \sum_{j \in J_+} \left[ \left( \sum_{i \in I^{(r)}} \lambda_i v_j^2 \phi'_{i,j} (\bx_i^\top \bx + 1) \right) + \sum_{q \in Q \setminus \{r\}} \sum_{i \in I^{(q)}} \lambda_i y_i v_j^2 \phi'_{i,j} (\bx_i^\top \bx + 1) \right]
    \\
    &\geq   \left( \sum_{i \in I^{(r)}} \sum_{j \in J_+} \lambda_i v_j^2 \phi'_{i,j} (\bx_i^\top \bx + 1) \right) - \sum_{q \in Q \setminus \{r\}} \sum_{i \in I^{(q)}} \sum_{j \in J_+} \lambda_i v_j^2 \phi'_{i,j} | \bx_i^\top \bx + 1|~.
\end{align*}
Since $\bx$ comes from cluster $r$ and the clusters are nearly orthogonal, the pairwise correlations $\bx_i^\top \bx$ will be large and positive when $i\in I^{(r)}$ but will be small in magnitude when 
$i \in I^{(q)}$ for $q \neq r$.
Thus, we can hope that this term will be large and positive if we can show that the quantity $\sum_{i\in I^{(r)}} \sum_{j\in J^+} \lambda_i v_j^2 \phi'_{i,j}$ is not too small relative to the quantity 
$\sum_{q \in Q \setminus \{r\}} \sum_{i \in I^{(q)}} \sum_{j \in J_+} \lambda_i v_j^2 \phi'_{i,j}$.
By similar arguments, in order to show the second term in~\eqref{eq:network.output} is not too negative, we need to understand how the quantity $\sum_{i\in I^{(q)}} \sum_{j\in J_-}  \lambda_i v_j^2 \phi'_{i,j}$ varies across different clusters $q \in Q$. 

The above sketch motivates a characterization of how the quantities
\[ \sum_{i\in I^{(q)}} \sum_{j\in J_+} \lambda_i v_j^2 \phi'_{i,j}, \quad \sum_{i\in I^{(q)}} \sum_{j\in J_-} \lambda_i v_j^2 \phi'_{i,j}\]
relate to each other for different clusters $q\in Q$.  
We will obtain upper and lower bounds for these quantities in Lemmas~\ref{lem:lam upper bound} and~\ref{lem:lam lower bound} below.
We now proceed with the proof.

Recall that $\Delta := 4 \sigma \sqrt d \ln(d)$.  Since by \assref{ass:dist} we have $k \leq \frac{d - \Delta + 1}{10(p+\Delta+1)}$, we let $c' \leq \frac{1}{10}$ be such that $k = c' \cdot \frac{d - \Delta + 1}{p+\Delta+1}$. 
Note that $d>\Delta$, and more precisely, the following holds
\begin{lemma} \label{lem:bound Delta}
    We have $\Delta \leq \frac{d}{21}$.
\end{lemma}
\begin{proof}
    Recall that $k \left( p + \Delta + 1 \right) \leq \frac{d - \Delta + 1 }{10}$. Since $k \geq 2$ and $p \geq 0$ it implies that $2(\Delta+1) \leq \frac{d - \Delta + 1 }{10}$. Hence, $\Delta \leq \frac{d-19}{21} \leq \frac{d}{21}$.
\end{proof}

We now show that the sums of the form $\sum_{i\in I^{(q)}} \sum_{j\in J_{\circ}} v_j^2 \lambda_i \phi'_{i,j}$, for $\circ \in \{+,-\}$, are never too large. 

\begin{lemma} \label{lem:lam upper bound}
	If $\cs$ satisfies the properties~\ref{p:noise.norm} through~\ref{p:sample.in.every.cluster}, then for all $q \in Q$ we have 
	\[
		\max\left\{ \sum_{i \in I^{(q)}} \sum_{j \in J_+} v_j^2 \lambda_i \phi'_{i,j}, \sum_{i \in I^{(q)}} \sum_{j \in J_-} v_j^2 \lambda_i \phi'_{i,j} \right\}
		\leq \frac{1}{(1-2c')(d-\Delta + 1)}~.
	\]
\stam{
	\[
		\sum_{i \in I^{(q)}} \sum_{j \in J_+} v_j^2 \lambda_i \phi'_{i,j} \leq \frac{1}{(1-2c')(d-\Delta + 1)}~,
	\]
	and 
	\[
		\sum_{i \in I^{(q)}} \sum_{j \in J_-} v_j^2 \lambda_i \phi'_{i,j} \leq \frac{1}{(1-2c')(d-\Delta + 1)}~.
	\]
}
\end{lemma}
\begin{proof}
	Let $\alpha_+ = \max_{q \in Q} \left( \sum_{i \in I^{(q)}} \sum_{j \in J_+} v_j^2 \lambda_i \phi'_{i,j} \right)$, and let  $\alpha_- = \max_{q \in Q} \left( \sum_{i \in I^{(q)}} \sum_{j \in J_-} v_j^2 \lambda_i \phi'_{i,j} \right)$. Assume w.l.o.g. that $\alpha_+ \geq \alpha_-$ (the proof for the case $\alpha_+ < \alpha_-$ is similar). Let $\alpha = \alpha_+$ and $r \in \argmax_{q \in Q} \left( \sum_{i \in I^{(q)}} \sum_{j \in J_+} v_j^2 \lambda_i \phi'_{i,j} \right)$.
	Assume towards contradiction that $\alpha > \frac{1}{(1-2c')(d-\Delta + 1)}$. 
	Note that we have $\sum_{i \in I^{(r)}} \lambda_i > 0$, since otherwise $\alpha = 0$. Hence, there exists $i' \in I^{(r)}$ with $\lambda_{i'} > 0$, and thus $y_{i'} \cn_\btheta(\bx_{i'}) = 1$.
	
	By \eqref{eq:kkt condition w} and~(\ref{eq:kkt condition b}) for every $j \in J$ we have 
	\begin{align} \label{eq:single neuron separating r}
		\bw_j^\top \bx_{i'} + b_j
		&=  \sum_{i \in I} \lambda_i y_i v_j \phi'_{i,j} \bx_i^\top \bx_{i'} + \sum_{i \in I} \lambda_i y_i v_j \phi'_{i,j} \nonumber
		\\
		&= \sum_{i \in I} \lambda_i y_i v_j \phi'_{i,j} (\bx_i^\top \bx_{i'} + 1) \nonumber
		\\
		&= \left( \sum_{q \in Q \setminus \{r\}} \sum_{i \in I^{(q)}} \lambda_i y_i v_j \phi'_{i,j} (\bx_i^\top \bx_{i'} + 1) \right) +  \sum_{i \in I^{(r)}} \lambda_i y_i v_j \phi'_{i,j} (\bx_i^\top \bx_{i'} + 1)~.
	\end{align}
	
	We consider two cases:
	
	{\bf Case 1:} Assume that $r \in Q_+$. We have
	\begin{align} \label{eq:positive clusters bound1}
		1
		&= y_{i'} \cn_\btheta(\bx_{i'}) \nonumber
		\\
		&= 1 \cdot \sum_{j \in J} v_j \phi(\bw_j^\top \bx_{i'} + b_j) \nonumber
		\\
		&\geq \sum_{j \in J_+} v_j (\bw_j^\top \bx_{i'} + b_j) + \sum_{j \in J_-} v_j \phi(\bw_j^\top \bx_{i'} + b_j)~.
	\end{align}
	By the case assumption $r\in Q_+$,~\eqref{eq:single neuron separating r} and our assumptions on the dataset $\cs$, we have
	\begin{align} \label{eq:positive clusters bound2}
		\sum_{j \in J_+} v_j (\bw_j^\top \bx_{i'} + b_j)
		&= \sum_{j \in J_+} \left[ \left( \sum_{q \in Q \setminus \{r\}} \sum_{i \in I^{(q)}} \lambda_i y_i v_j^2 \phi'_{i,j} (\bx_i^\top \bx_{i'} + 1) \right) +  \sum_{i \in I^{(r)}} \lambda_i y_i v_j^2 \phi'_{i,j} (\bx_i^\top \bx_{i'} + 1) \right]	\nonumber
		\\
		&\geq \sum_{j \in J_+} \left[ \left( - \sum_{q \in Q \setminus \{r\}} \sum_{i \in I^{(q)}} \lambda_i v_j^2 \phi'_{i,j} (p + \Delta + 1) \right) +  \sum_{i \in I^{(r)}} \lambda_i v_j^2 \phi'_{i,j} (d - \Delta + 1) \right] \nonumber
		\\
		&= \left( - (p + \Delta + 1) \sum_{q \in Q \setminus \{r\}} \sum_{i \in I^{(q)}} \sum_{j \in J_+}  \lambda_i v_j^2 \phi'_{i,j} \right) +  (d - \Delta + 1)  \sum_{i \in I^{(r)}} \sum_{j \in J_+} \lambda_i v_j^2 \phi'_{i,j} \nonumber
		\\
		&\geq - (p + \Delta + 1) k \alpha +  (d - \Delta + 1) \alpha~.
	\end{align}
In the last line we have used the definition $\alpha = \max_{q \in Q} \left( \sum_{i \in I^{(q)}} \sum_{j \in J_+} v_j^2 \lambda_i \phi'_{i,j} \right)$ and that the cluster with index $r$ achieves this maximum.  
	Moreover, using \eqref{eq:single neuron separating r} again we have
	\begin{align} \label{eq:positive clusters bound3}
		 \sum_{j \in J_-} v_j \phi(\bw_j^\top \bx_{i'} + b_j)
		 &= \sum_{j \in J_-} v_j \phi\left(  \left( \sum_{q \in Q \setminus \{r\}} \sum_{i \in I^{(q)}} \lambda_i y_i v_j \phi'_{i,j} (\bx_i^\top \bx_{i'} + 1) \right) +  \sum_{i \in I^{(r)}} \lambda_i y_i v_j \phi'_{i,j} (\bx_i^\top \bx_{i'} + 1) \right) \nonumber
		 \\
		 &\geq \sum_{j \in J_-} v_j \phi\left(  \left( \sum_{q \in Q \setminus \{r\}} \sum_{i \in I^{(q)}} \lambda_i | v_j | \phi'_{i,j} (p + \Delta + 1) \right) +  \sum_{i \in I^{(r)}} \lambda_i v_j \phi'_{i,j} (d - \Delta + 1) \right) \nonumber
		 \\
		 &\geq \sum_{j \in J_-} v_j \phi\left(  \sum_{q \in Q \setminus \{r\}} \sum_{i \in I^{(q)}} \lambda_i | v_j | \phi'_{i,j} (p + \Delta + 1) \right) \nonumber
		 \\
		 &= \sum_{j \in J_-} v_j \left( \sum_{q \in Q \setminus \{r\}} \sum_{i \in I^{(q)}} \lambda_i | v_j | \phi'_{i,j} (p + \Delta + 1) \right)  \nonumber
		 \\
		 &= -(p + \Delta + 1)  \left( \sum_{q \in Q \setminus \{r\}} \sum_{i \in I^{(q)}} \sum_{j \in J_-}  \lambda_i  v_j^2  \phi'_{i,j} \right)  \nonumber
		 \\
		 &\geq -(p + \Delta + 1)  k \alpha~.
	\end{align}
	The first inequality above uses the properties of the dataset $\cs$, that $j\in J_{-}$ and $\phi$ is non-decreasing, as well as the case assumption that $r\in Q_+$.  The last inequality uses the definition of $\alpha$.  
	Combining \eqref{eq:positive clusters bound1},~(\ref{eq:positive clusters bound2}), and~(\ref{eq:positive clusters bound3}) we get
	\begin{align*}
		1
		&\geq - (p + \Delta + 1) k \alpha + (d - \Delta + 1) \alpha - (p + \Delta + 1)  k \alpha 
		\\
		&= \alpha \left( (d - \Delta + 1) -2k (p + \Delta + 1) \right)
		\\		
		&= \alpha \left( (d - \Delta + 1) -2 \cdot \frac{c' (d - \Delta + 1)}{p+\Delta+ 1}  (p + \Delta + 1) \right)
		\\
		&= \alpha \left( (d - \Delta + 1) -2 c' (d - \Delta + 1)  \right)
		\\
		&= \alpha (d - \Delta + 1) (1-2c')
		\\
		&> \frac{1}{(1-2c')(d-\Delta + 1)}  (d - \Delta + 1)(1-2c') 
		\\
		&= 1~,
	\end{align*}
	where in the last inequality we used our assumption on $\alpha$.  We have thus reached a contradiction following our assumption on $\alpha$ in the case where $r\in Q_+$. 
	
	{\bf Case 2:} Assume that $r \in Q_-$. Fix some $j \in J_+$. If $\phi'_{i,j} = 0$ for every $i \in I^{(r)}$ then 
	\begin{equation} \label{eq:case 2 simple case}
		\sum_{i \in I^{(r)}} \lambda_i v_j \phi'_{i,j} (d - \Delta + 1) = 0 \leq \sum_{q \in Q \setminus \{r\}} \sum_{i \in I^{(q)}} \lambda_i v_j \phi'_{i,j} (p + \Delta + 1)~.
	\end{equation}
	Otherwise, i.e., if there is $s \in  I^{(r)}$ such that $\phi'_{s,j} > 0$, then by the definition of $\phi'_{s,j}$ we have $\bw_j^\top \bx_s + b_j \geq 0$, and hence by \eqref{eq:single neuron separating r} we have
	\begin{align*}
		0 
		&\leq \bw_j^\top \bx_s + b_j
		\\
		&= \left( \sum_{q \in Q \setminus \{r\}} \sum_{i \in I^{(q)}} \lambda_i y_i v_j \phi'_{i,j} (\bx_i^\top \bx_{s} + 1) \right) +  \sum_{i \in I^{(r)}} \lambda_i y_i v_j \phi'_{i,j} (\bx_i^\top \bx_{s} + 1)
		\\
		&\leq \left( \sum_{q \in Q \setminus \{r\}} \sum_{i \in I^{(q)}} \lambda_i v_j \phi'_{i,j} (p + \Delta + 1) \right) - \sum_{i \in I^{(r)}} \lambda_i v_j \phi'_{i,j} (d - \Delta + 1)
	\end{align*}
	Hence, we get again an expression similar to \eqref{eq:case 2 simple case}. Thus for any $j\in J_+$ we have,
	\[
		\sum_{i \in I^{(r)}} \lambda_i v_j \phi'_{i,j} 
		\leq \frac{ p + \Delta + 1 }{d - \Delta + 1} \left( \sum_{q \in Q \setminus \{r\}} \sum_{i \in I^{(q)}} \lambda_i v_j \phi'_{i,j} \right)~.
	\]
	Since this holds for every $j \in J_+$, we get
	\begin{align*}
		\sum_{i \in I^{(r)}} \sum_{j \in J_+}   \lambda_i v_j^2 \phi'_{i,j}
		&= \sum_{j \in J_+} v_j \sum_{i \in I^{(r)}} \lambda_i v_j \phi'_{i,j}
		\\
		&\leq \sum_{j \in J_+} v_j \cdot \frac{p + \Delta + 1}{d - \Delta + 1} \left( \sum_{q \in Q \setminus \{r\}} \sum_{i \in I^{(q)}} \lambda_i v_j \phi'_{i,j} \right)
		\\
		&= \frac{p + \Delta + 1}{d - \Delta + 1} \sum_{q \in Q \setminus \{r\}} \sum_{i \in I^{(q)}} \sum_{j \in J_+} \lambda_i v_j^2 \phi'_{i,j}  
		\\
		&\leq \frac{p + \Delta + 1}{d - \Delta + 1} \cdot k \cdot \max_{q \in Q} \left( \sum_{i \in I^{(q)}} \sum_{j \in J_+} \lambda_i v_j^2 \phi'_{i,j} \right)
		\\
		&= \frac{p + \Delta + 1}{d - \Delta + 1} \cdot \frac{c'(d-\Delta+1)}{p+\Delta+1} \cdot \max_{q \in Q} \left(\sum_{i \in I^{(q)}} \sum_{j \in J_+} \lambda_i v_j^2 \phi'_{i,j} \right)
		\\
		&< \max_{q \in Q} \left( \sum_{i \in I^{(q)}} \sum_{j \in J_+} \lambda_i v_j^2 \phi'_{i,j} \right). 
	\end{align*}
Since $r \in \argmax_{q \in Q} \left( \sum_{i \in I^{(q)}} \sum_{j \in J_+} v_j^2 \lambda_i \phi'_{i,j} \right)$, we have reached a contradiction following our assumption on $\alpha$ for the case $r\in Q_{-}$.  This completes the proof that we must have $\alpha \leq \frac{1}{(1-2c')(d-\Delta+1)}$. 
\end{proof}

We next show that the relevant sums of the form $\sum_{i\in I^{(q)}} \sum_{j\in J_{\circ}} v_j^2 \lambda_i \phi'_{i,j}$, for $\circ \in \{+,-\}$, are never too small. 

\begin{lemma} \label{lem:lam lower bound}
	If $\cs$ satisfies the properties~\ref{p:noise.norm} through~\ref{p:sample.in.every.cluster}, then for all $q \in Q_+$ we have 
	\[
		\sum_{i \in I^{(q)}} \sum_{j \in J_+} v_j^2 \lambda_i \phi'_{i,j} \geq \left( 1- \frac{c'}{1-2c'} \right) \frac{1}{3d+\Delta + 1}~,
	\]
	and for all $q \in Q_-$ we have 
	\[
		\sum_{i \in I^{(q)}} \sum_{j \in J_-} v_j^2 \lambda_i \phi'_{i,j} \geq \left( 1- \frac{c'}{1-2c'} \right) \frac{1}{3d+\Delta + 1}~.
	\]
\end{lemma}
\begin{proof}
	Let $r \in Q_+$ and let $s \in I^{(r)}$. We have
	\[
		1 
		\leq \cn_\btheta(\bx_s) 
		= \sum_{j \in J} v_j \phi(\bw_j^\top \bx_s + b_j) 
		\leq \sum_{j \in J_+} v_j \phi(\bw_j^\top \bx_s + b_j) 
		\leq \sum_{j \in J_+} v_j \left| \bw_j^\top \bx_s + b_j \right|~.
	\]
	By \eqref{eq:kkt condition w} and~(\ref{eq:kkt condition b}), since $r\in Q_+$ the above equals
	\begin{align*}
		\sum_{j \in J_+} v_j & \left|  \sum_{i \in I} \left( \lambda_i y_i v_j \phi'_{i,j} \bx_i^\top \bx_s + \lambda_i y_i v_j \phi'_{i,j} \right) \right|
		\leq \sum_{j \in J_+} v_j \sum_{q \in Q} \sum_{i \in I^{(q)}} \left| \lambda_i y_i v_j \phi'_{i,j} (\bx_i^\top \bx_s + 1)\right|
		\\
		&=  \sum_{j \in J_+} v_j \left[ \left( \sum_{i \in I^{(r)}}  \lambda_i v_j \phi'_{i,j} \left| \bx_i^\top \bx_s + 1 \right| \right) + \sum_{q \in Q \setminus \{r\}} \sum_{i \in I^{(q)}}  \lambda_i v_j \phi'_{i,j} \left| \bx_i^\top \bx_s + 1 \right| \right] 
		\\
		&= \left( \sum_{i \in I^{(r)}} \sum_{j \in J_+}  \lambda_i v_j^2 \phi'_{i,j} \left| \bx_i^\top \bx_s + 1 \right| \right)
		  + \sum_{q \in Q \setminus \{r\}} \sum_{i \in I^{(q)}}  \sum_{j \in J_+}  \lambda_i v_j^2 \phi'_{i,j} \left| \bx_i^\top \bx_s + 1 \right|
		\\
		&\leq \left( (3d + \Delta + 1) \sum_{i \in I^{(r)}} \sum_{j \in J_+}  \lambda_i v_j^2 \phi'_{i,j} \right)
		  + (p +\Delta + 1)\sum_{q \in Q \setminus \{r\}} \sum_{i \in I^{(q)}}  \sum_{j \in J_+}  \lambda_i v_j^2 \phi'_{i,j} 
	\end{align*}
The final inequality uses the properties of the dataset $\cs$.	Combining the above with \lemref{lem:lam upper bound} we get
	\begin{align*}
		1 
		&\leq \left( (3d + \Delta + 1) \sum_{i \in I^{(r)}} \sum_{j \in J_+}  \lambda_i v_j^2 \phi'_{i,j} \right) + (p +\Delta + 1) k \cdot \frac{1}{(1-2c')(d-\Delta + 1)}
		\\
		&= \left( (3d + \Delta + 1) \sum_{i \in I^{(r)}} \sum_{j \in J_+}  \lambda_i v_j^2 \phi'_{i,j} \right) + (p +\Delta + 1) \cdot \frac{c'(d-\Delta + 1)}{p+\Delta+1} \cdot \frac{1}{(1-2c')(d-\Delta + 1)} 
		\\
		&= \left( (3d + \Delta + 1) \sum_{i \in I^{(r)}} \sum_{j \in J_+}  \lambda_i v_j^2 \phi'_{i,j} \right) +\frac{c'}{1-2c'}~. 
	\end{align*}
	Therefore,
	\[
		 \sum_{i \in I^{(r)}} \sum_{j \in J_+}  \lambda_i v_j^2 \phi'_{i,j} 
		 \geq \left( 1- \frac{c'}{1-2c'} \right) \frac{1}{3d+\Delta + 1}~.
	\]
	By similar arguments with $r \in Q_-$ we also get 
	\[
		 \sum_{i \in I^{(r)}} \sum_{j \in J_-}  \lambda_i v_j^2 \phi'_{i,j} 
		 \geq \left( 1- \frac{c'}{1-2c'} \right) \frac{1}{3d+\Delta + 1}~.
	\]
\end{proof}

Recall that test examples will come from one of the $k$ nearly-orthogonal clusters.  Since the clusters are nearly-orthogonal, the pairwise correlations between the test example and training data from the same cluster will be much larger than the pairwise correlations between the test example and training data from the other (nearly-orthogonal) clusters.  To characterize the decision boundary of the neural network on test data it therefore suffices to characterize the decision boundary for an example $\bx$ that is (1) highly correlated to examples from a given cluster and (2) nearly-orthogonal to samples from other clusters.    The next lemma leverages the structural conditions provided in Lemmas~\ref{lem:lam upper bound} and~\ref{lem:lam lower bound} to show exactly this.
\begin{lemma} \label{lem:good example labeld correctly}
Suppose $\cs$ satisfies the properties~\ref{p:noise.norm} through~\ref{p:sample.in.every.cluster}. 	Let $\bx \in \reals^d$ and $r \in Q$ be such that for all $i \in I^{(r)}$ we have $\inner{\bx,\bx_i} \in \left[d - \Delta, d + \Delta \right]$, and for all $q \in Q \setminus \{r\}$ and $i \in I^{(q)}$ we have $| \inner{\bx,\bx_i} | \leq p + \Delta$.  Then, $\sign\left(\cn_{\btheta}(\bx)\right) = y^{(r)}$.
\end{lemma}
\begin{proof}
	We prove the claim for $r \in Q_+$. The proof for $r \in Q_-$ is similar.
	By \eqref{eq:kkt condition w} and~(\ref{eq:kkt condition b}), for every $j \in J$ we have 
	\begin{align} \label{eq:single neuron separating r 2}
		\bw_j^\top \bx + b_j
		&= \left( \sum_{i \in I} \lambda_i y_i v_j \phi'_{i,j} \bx_i^\top \bx \right) + \sum_{i \in I} \lambda_i y_i v_j \phi'_{i,j} \nonumber
		\\
		&= \sum_{i \in I} \lambda_i y_i v_j \phi'_{i,j} (\bx_i^\top \bx + 1) \nonumber
		\\
		&= \left( \sum_{i \in I^{(r)}} \lambda_i v_j \phi'_{i,j} (\bx_i^\top \bx + 1)  \right) + \sum_{q \in Q \setminus \{r\}} \sum_{i \in I^{(q)}} \lambda_i y_i v_j \phi'_{i,j} (\bx_i^\top \bx + 1)~.
	\end{align}

	Now,
	\begin{align} \label{eq:general neurons sum}
		\cn_\btheta(\bx)
		&= \sum_{j \in J} v_j \phi(\bw_j^\top \bx + b_j)
		\geq \sum_{j \in J_+} v_j (\bw_j^\top \bx + b_j) + \sum_{j \in J_-} v_j \phi(\bw_j^\top \bx + b_j)~.
	\end{align}
	By \eqref{eq:single neuron separating r 2} we have
	\begin{align*}
		\sum_{j \in J_+} v_j (\bw_j^\top \bx + b_j) 
		&= \sum_{j \in J_+} \left[ \left( \sum_{i \in I^{(r)}} \lambda_i v_j^2 \phi'_{i,j} (\bx_i^\top \bx + 1) \right) + \sum_{q \in Q \setminus \{r\}} \sum_{i \in I^{(q)}} \lambda_i y_i v_j^2 \phi'_{i,j} (\bx_i^\top \bx + 1) \right]
		\\
		&\geq \sum_{j \in J_+} \left[ \left( \sum_{i \in I^{(r)}} \lambda_i v_j^2 \phi'_{i,j} (d - \Delta + 1) \right) - \sum_{q \in Q \setminus \{r\}} \sum_{i \in I^{(q)}} \lambda_i v_j^2 \phi'_{i,j} (p + \Delta + 1) \right]
		\\
		&= \left( (d - \Delta + 1) \sum_{i \in I^{(r)}} \sum_{j \in J_+} \lambda_i v_j^2 \phi'_{i,j}  \right) 
		\\
		&\;\;\;\; - \left(  (p + \Delta + 1) \sum_{q \in Q \setminus \{r\}} \sum_{i \in I^{(q)}} \sum_{j \in J_+} \lambda_i v_j^2 \phi'_{i,j} \right)~.
	\end{align*}
	By \lemref{lem:lam upper bound} and \lemref{lem:lam lower bound} the above is at least
	\begin{align} \label{eq:general neurons sum part 1}
		(d - \Delta + 1) &\left( 1- \frac{c'}{1-2c'} \right) \frac{1}{3d+\Delta + 1} -  (p + \Delta + 1) k \cdot \frac{1}{(1-2c')(d-\Delta + 1)} \nonumber
		\\
		&=\left( 1- \frac{c'}{1-2c'} \right) \frac{d - \Delta + 1}{3d+\Delta + 1} -  (p + \Delta + 1) c' \cdot \frac{d - \Delta + 1}{p+\Delta+1} \cdot  \frac{1}{(1-2c')(d-\Delta + 1)} \nonumber
		\\
		&= \left( 1- \frac{c'}{1-2c'} \right) \frac{d - \Delta + 1}{3d+\Delta + 1} - \frac{c'}{1-2c'}~.
	\end{align}
	
	Likewise, we have 
	\begin{align*}
		 \sum_{j \in J_-} v_j \phi(\bw_j^\top \bx + b_j)
		 &=  \sum_{j \in J_-} v_j \phi \left(  \left( \sum_{i \in I^{(r)}} \lambda_i v_j \phi'_{i,j} (\bx_i^\top \bx + 1)  \right) + \sum_{q \in Q \setminus \{r\}} \sum_{i \in I^{(q)}} \lambda_i y_i v_j \phi'_{i,j} (\bx_i^\top \bx + 1) \right)
		 \\
		 &\geq  \sum_{j \in J_-} v_j \phi \left(  \left( \sum_{i \in I^{(r)}} \lambda_i v_j \phi'_{i,j} (d - \Delta + 1)  \right) + \sum_{q \in Q \setminus \{r\}} \sum_{i \in I^{(q)}} \lambda_i | v_j | \phi'_{i,j} (p + \Delta + 1) \right)
		 \\
		 &\geq \sum_{j \in J_-} v_j \phi \left( \sum_{q \in Q \setminus \{r\}} \sum_{i \in I^{(q)}} \lambda_i | v_j | \phi'_{i,j} (p + \Delta + 1) \right)
		 \\
		 &= - \sum_{j \in J_-}  \sum_{q \in Q \setminus \{r\}} \sum_{i \in I^{(q)}} \lambda_i v_j^2 \phi'_{i,j} (p + \Delta + 1)
		 \\
		 &= -  (p + \Delta + 1)  \sum_{q \in Q \setminus \{r\}}  \sum_{i \in I^{(q)}} \sum_{j \in J_-} \lambda_i v_j^2 \phi'_{i,j}~.
	\end{align*}
	By  \lemref{lem:lam upper bound} the above is at least
	\begin{align} \label{eq:general neurons sum part 2}
		- (p + \Delta + 1) k \cdot  \frac{1}{(1-2c')(d-\Delta + 1)}
		&=- (p + \Delta + 1) c' \cdot \frac{d - \Delta + 1}{p+\Delta+1} \cdot  \frac{1}{(1-2c')(d-\Delta + 1)} \nonumber
		\\
		&= - \frac{c'}{1-2c'}~. 
	\end{align}
	Combining \eqref{eq:general neurons sum},~(\ref{eq:general neurons sum part 1}), and~(\ref{eq:general neurons sum part 2}), we get
	\begin{align*}
		\cn_\btheta(\bx)
		&\geq  \left( 1- \frac{c'}{1-2c'} \right) \frac{d - \Delta + 1}{3d+\Delta + 1} - \frac{c'}{1-2c'} - \frac{c'}{1-2c'}~. 
	\end{align*}	
	Using $c' \leq \frac{1}{10}$ and $\Delta \leq d$ (which holds by \lemref{lem:bound Delta}),  the above is at least 
	\[
		\frac{7}{8} \cdot \frac{d - \Delta + 1}{3d+\Delta + 1} - \frac{2}{8} 
		\geq \frac{7}{8} \cdot \frac{d - \Delta}{3d+\Delta} - \frac{2}{8}~.
	\]
	By \lemref{lem:bound Delta},
	the displayed equation is at least
	\[
		\frac{7}{8} \cdot \frac{d - d/21}{3d+d/21} - \frac{2}{8}
		= \frac{7}{8} \cdot \frac{5}{16} - \frac{2}{8}
		> 0~.
	\]	
\end{proof}

\subsection{Generalization from KKT conditions}
\label{app:generalization.conclusion}
Lemma~\ref{lem:good example labeld correctly} shows that in order to show generalization, it suffices to show that with high probability, a test example is highly correlated to one cluster and nearly-orthogonal to all other clusters.  In this section we prove that this is the case.  We shall re-apply many of the concentration bounds provided in Section~\ref{app:generalization.nicedata} to do so.

\begin{lemma} \label{lem:random example is good}
	Suppose $\cs$ satisfies Properties~\ref{p:noise.norm} through~\ref{p:sample.in.every.cluster}.  
 Let $r \in Q$ and let $\bx = \bmu^{(r)} + \bxi$ where $\bxi \sim \ndist(\zero, \sigma^2 I_d)$.
	With probability at least $1 - 4 n d^{-\frac{\ln(d)}{2}}$ over $\bx$ the following hold for all $i \in I$:
	\begin{itemize}
		\item $| \inner{\bxi_i,\bxi}| \leq \Delta$.
		\item $| \inner{\bx_i,\bxi}| \leq \Delta$. 
		\item If $i \not \in I^{(r)}$ then $|\inner{\bx_i,\bx}| \leq p + \Delta$.
		\item If $i \in I^{(r)}$ then $| \inner{\bx_i,\bx} - d | \leq \Delta$.
	\end{itemize}
\end{lemma}
\begin{proof}
	By our assumption on the dataset $\cs$, for all $i \in I$ and $q \in Q$ we have $\norm{\bxi_i} \leq \sigma \sqrt{2d}$ and $\inner{\bmu^{(q)},\bxi_i} \leq \sigma \sqrt{d} \ln(d)$. 
	By \lemref{lem:bound xis product fixed} and since $\sigma \leq 1$, for $i \in I$ we have 
	\[
		\Pr\left[ | \inner{\bxi, \bxi_i} | \geq 2\sigma \sqrt{d} \ln(d) \right] 
		\leq \Pr\left[ | \inner{\bxi, \bxi_i} | \geq \sigma \sqrt{2d} \cdot \sigma \ln(d) \right] 
		\leq \Pr\left[ | \inner{\bxi, \bxi_i} | \geq \norm{\bxi_i} \sigma \ln(d) \right] 
		\leq 2d^{-\frac{\ln(d)}{2}}~,
	\]
	and by \lemref{lem:bound mu xi product} for $q \in Q$ we have 
	\[
		\Pr \left[ |\inner{\bmu^{(q)}, \bxi} | \geq \sigma \sqrt{d} \ln(d) \right] 
		\leq 2d^{-\frac{\ln(d)}{2}}~.
	\]
	Fix some $q \in Q$ and $i \in I^{(q)}$. With probability at least $1 - 4 d^{-\frac{\ln(d)}{2}}$ over $\bxi$, we have $| \inner{\bxi, \bxi_i} | \leq 2\sigma \sqrt{d} \ln(d)$ and $ |\inner{\bmu^{(q)}, \bxi} | \leq \sigma \sqrt{d} \ln(d)$. 
	
	Then, the following hold:
	\begin{itemize}
		\item $| \inner{\bxi_i,\bxi}| \leq \Delta$.
		\item We have
			\[
				| \inner{\bx_i,\bxi} | \leq 
				| \inner{\bmu^{(q)}, \bxi} | + | \inner{\bxi_i, \bxi} |
				\leq \sigma \sqrt{d} \ln(d) + 2\sigma \sqrt{d} \ln(d)
				\leq \Delta~. 
			\]
		\item If $q \neq r$, then 
			\begin{align*}
				|\inner{\bx,\bx_i}| 
				&= |\inner{\bmu^{(r)} + \bxi, \bmu^{(q)} + \bxi_i } |
				\\
				&\leq  |\inner{\bmu^{(r)}, \bmu^{(q)}} | + |\inner{\bmu^{(r)}, \bxi_i } | + |\inner{\bxi, \bmu^{(q)}} | + |\inner{\bxi, \bxi_i } |
				 \\
				 &\leq p +  \sigma \sqrt{d} \ln(d) +\sigma \sqrt{d} \ln(d) + 2\sigma \sqrt{d} \ln(d)
				 \\
				 &= p + \Delta~.
			\end{align*}
		\item If $q = r$, then	
			\begin{align*}
				| \inner{\bx, \bx_i} - d |
				&= \left| \inner{\bmu^{(r)} + \bxi, \bmu^{(r)} + \bxi_i } - d \right|
				\\
				&= \left| \inner{\bmu^{(r)}, \bxi_i } +\inner{\bxi, \bmu^{(r)}} + \inner{\bxi, \bxi_i }  \right|  
				\\
				&\leq | \inner{\bmu^{(r)}, \bxi_i } | + | \inner{\bxi, \bmu^{(q)}} | + | \inner{\bxi, \bxi_i } |
				\\
				&\leq \sigma \sqrt{d} \ln(d) + \sigma \sqrt{d} \ln(d)  + 2 \sigma \sqrt{d} \ln(d)
				\\
				&= \Delta~.
	\end{align*}
			
	\end{itemize}	
	
	Overall, by the union bound, with probability at least $1 - 4 n d^{-\frac{\ln(d)}{2}}$ the requirements hold for all $i \in I$.	
\end{proof}

Theorem~\ref{thm:generalizaton orig} now follows immediately from Proposition~\ref{prop:trainingdata} and Lemmas~\ref{lem:good example labeld correctly} and~\ref{lem:random example is good}.


\section{Proof of \lemref{lem:linearly separable}} \label{app:proof of linearly separable}

Let $\bx  = \bmu^{(j)} + \bxi$ where $\bxi \sim \ndist(\zero, \sigma^2 I_d)$, and $y = y^{(j)}$. Then, 
\[
	y \bu^\top \bx 
	=  y^{(j)} \sum_{q \in [k]} y^{(q)} (\bmu^{(q)})^\top (\bmu^{(j)} + \bxi)
	\geq \norm{\bmu^{(j)}}^2 - k \left( \max_{q \neq j} |(\bmu^{(q)})^\top \bmu^{(j)} | + \max_{q \in [k]} | (\bmu^{(q)})^\top \bxi | \right)~.
\]
By  
\lemref{lem:bound mu xi product} we have $\Pr \left[ |\inner{\bmu^{(q)}, \bxi} | \geq \sigma \sqrt{d} \ln(d) \right] \leq 2d^{-\frac{\ln(d)}{2}}$.
Hence, by the union bound with probability at least $1- 2kd^{-\ln(d)/2} \geq 1- 2d^{1-\ln(d)/2}$ we have $\max_{q \in [k]} | (\bmu^{(q)})^\top \bxi | \leq \sigma \sqrt{d} \ln(d)$. Note that by \assref{ass:dist} we must have $k \leq d$.
Using \assref{ass:dist} again, we conclude that with probability at least $1 - 2d^{1-\ln(d)/2}$ we have  
\[
	y \bu^\top \bx 
	\geq d - k  \left( \max_{q \neq j} |(\bmu^{(q)})^\top \bmu^{(j)} | +  \sigma \sqrt{d} \ln(d) \right)
	\geq d -  \frac{d  - 4\sigma \sqrt{d} \ln(d) + 1 }{10}
	\geq \frac{9d - 1}{10}
	> 0~.
\]

\section{Proof of \thmref{thm:robust exists}} \label{app:proof of robust exists}

We prove the theorem for $r=k$. The proof for $r > k$ follows immediately by adding zero-weight neurons.
Consider the network $\cn(\bx) = \sum_{j=1}^k v_j \phi(\bw_j^\top \bx + b_j)$ such that for every $j \in [k]$ we have $v_j=y^{(j)}$, $\bw_j = \frac{4 \bmu^{(j)}}{d}$, and $b_j = - 2$. Let $q \sim \cu([k])$, let $\bx = \bmu^{(q)} + \bxi$ where $\bxi \sim \ndist(\zero, \sigma^2 I_d)$, and let $y=y^{(q)}$. Thus $(\bx,y)$ is drawn from $\Dclust$.  We have,
	\[
		\bw_q^\top \bx + b_q
		= \frac{4 (\bmu^{(q)})^\top ( \bmu^{(q)} + \bxi)}{d} - 2
		= \frac{4 (d + \inner{\bmu^{(q)}, \bxi})}{d} - 2~.
	\] 
	By \lemref{lem:bound mu xi product}, we have $\Pr \left[ |\inner{\bmu^{(q)}, \bxi} | \geq \sigma \sqrt{d} \ln(d) \right] \leq 2d^{-\frac{\ln(d)}{2}}$. Hence, with probability at least $1 - 2d^{-\frac{\ln(d)}{2}}$ over $\bxi$, for large enough $d$ we have
	\begin{equation} \label{eq:good neuron is good}
		\bw_q^\top \bx + b_q
		\geq \frac{4 (d - \sigma \sqrt{d} \ln(d))}{d} - 2
		\geq \frac{4 (d - \sqrt{d} \ln(d))}{d} - 2
		= 2 - \frac{4 \ln(d)}{\sqrt{d}}
		\geq 1~.
	\end{equation}
	
	For $j \neq q$ we have
	\[
		\bw_j^\top \bx + b_j
		= \frac{4 (\bmu^{(j)})^\top ( \bmu^{(q)} + \bxi)}{d} - 2
		\leq \frac{4 \max_{j \neq q} | \inner{\bmu^{(j)},\bmu^{(q)}}|}{d} + \frac{4(\bmu^{(j)})^\top \bxi}{d} - 2~,
	\]
	and thus with probability at least  $1- 2d^{-\frac{\ln(d)}{2}}$ over $\bxi$ we have
	\begin{equation*} 
		\bw_j^\top \bx + b_j 
		\leq -2 + \frac{4}{d} \left(  \max_{j \neq q} | \inner{\bmu^{(j)},\bmu^{(q)}}| + \sigma \sqrt{d} \ln(d) \right)~.
	\end{equation*}
	Since by \assref{ass:dist} we have $k \left( \max_{j \neq q} |\inner{\bmu^{(j)}, \bmu^{(q)}} | + 4\sigma \sqrt{d} \ln(d) + 1 \right) \leq \frac{d  - 4\sigma \sqrt{d} \ln(d) + 1 }{10}$, then the displayed equation is at most
	\begin{equation} \label{eq:bad neuron is bad}
	    -2 + \frac{4}{d} \cdot \frac{d+1}{10k}
		\leq -2 + \frac{4}{d} \cdot \frac{2d}{10}
		\leq -1~.
	\end{equation}
	Overall, by the union bound, with probability at least  $1- 2kd^{-\frac{\ln(d)}{2}} \geq 1- 2d^{1-\frac{\ln(d)}{2}} = 1-o_d(1)$ we have $\cn(\bx) = \sum_{j=1}^k v_j \phi(\bw_j^\top \bx + b_j) = v_q (\bw_q^\top \bx + b_q) + 0$, and 
	\[
		\sign(\cn(\bx)) = \sign(v_q) = y^{(q)} = y~.
	\] 
	
	We now prove that $\cn$ is $\sqrt{d}$-robust w.r.t. $\cd_\bx$. Thus, we show that with probability at least $1-o_d(1)$ over $\bx$, for every $\bx' \in \reals^d$ such that $\norm{\bx-\bx'} \leq \frac{\sqrt{d}}{8}$ we have $\sign(\cn(\bx)) = \sign(\cn(\bx'))$. Note that with probability $1-o_d(1)$, \eqref{eq:good neuron is good} holds, and \eqref{eq:bad neuron is bad} holds for all $j \neq q$, and hence  
	\[
		\bw_q^\top \bx' + b_q
		= \bw_q^\top (\bx'-\bx) + \left( \bw_q^\top \bx + b_q \right)
		\geq - \norm{\bw_q} \cdot \norm{\bx'-\bx} + 1
		\geq  - \frac{4}{\sqrt{d}} \cdot \frac{\sqrt{d}}{8} + 1 
		= \frac{1}{2}~,
	\]
	and for all $j \neq q$ we have
	\[
		\bw_j^\top \bx' + b_j
		= \bw_j^\top (\bx'-\bx) + \left( \bw_j^\top \bx + b_j \right)
		\leq \norm{\bw_j} \cdot \norm{\bx'-\bx} - 1
		\leq \frac{4}{\sqrt{d}} \cdot \frac{\sqrt{d}}{8} - 1
		= -\frac{1}{2}~. 
	\]
	Therefore, $\sign(\cn(\bx')) = \sign(v_q) = y^{(q)} = \sign(\cn(\bx))$.	

\section{Proof of \thmref{thm:non-robust}} \label{app:proof of non-robust}

We will prove the following theorem:

\begin{theorem} \label{thm:non-robust orig}
	Let $\cs = \{(\bx_i,y_i)\}_{i=1}^n \subseteq \reals^d \times \{-1,1\}$ be a training set drawn i.i.d. from the distribution $\Dclust$, where $n \geq k \ln^2(d)$. 
	We denote $Q_+ = \{q \in [k]: y^{(q)} = 1\}$ and $Q_- = \{q \in [k]: y^{(q)} = -1\}$, and assume that $\min\left\{\frac{|Q_+|}{k},\frac{|Q_-|}{k} \right\} \geq c$ for some $c > 0$.
	Let $\cn_\btheta$ be a depth-$2$ ReLU network such that $\btheta = [\bw_1,\ldots,\bw_m,\bb,\bv]$ is a KKT point of Problem~(\ref{eq:optimization problem}). 
	Then, 
        with probability at least
	\[
		1- \left(3 n^2 d^{-\frac{\ln(d)}{2}} + n^2 e^{-d/16} + d^{1-\ln(d)} \right) 
	\]	
	over $\cs$, there is a vector $\bz = \eta \cdot \sum_{j \in [k]} y^{(j)} \bmu^{(j)}$ with $\eta > 0$ and $\norm{\bz} \leq \co\left( \sqrt{\frac{d}{c^2 k}} \right)$, such that 
	\[
		\Pr_{(\bx,y) \sim \Dclust} \left[ \sign(\cn_\btheta(\bx)) \neq \sign(\cn_\btheta(\bx - y \bz)) \right] 
		\geq 1 - 4 n d^{-\frac{\ln(d)}{2}}~.
	\]
\end{theorem}

It is easy to verify that \thmref{thm:non-robust orig} implies \thmref{thm:non-robust}. Indeed, if 
$\frac{1}{\delta} \leq \frac{1}{3} d^{\ln(d) - 1}$
and 
\[
    n \leq \min\left\{ \sqrt{\frac{\delta}{3}} \cdot e^{d/32}, \frac{\sqrt{\delta}}{3} \cdot d^{\ln(d)/4}, \frac{\epsilon}{4} \cdot d^{\ln(d)/2} \right\}~,
\]    
then we showed in \appref{app:generalization} that 
\begin{itemize}
    \item $3 n^2 d^{-\frac{\ln(d)}{2}} \leq \frac{\delta}{3}$.
    \item $n^2 e^{-d/16} \leq \frac{\delta}{3}$.
    \item $d^{1-\ln(d)} \leq \frac{\delta}{3}$.
    \item $4 n d^{-\frac{\ln(d)}{2}} \leq \epsilon$.
\end{itemize}
Hence, under the above assumptions on $\delta$ and $n$, \thmref{thm:non-robust orig} implies that with probability at least $1-\delta$ over $\cs$, we have $\Pr_{(\bx,y) \sim \Dclust} \left[ \sign(\cn_\btheta(\bx)) \neq \sign(\cn_\btheta(\bx - y \bz)) \right] \geq 1 - \epsilon$.

We now turn to prove \thmref{thm:non-robust orig}. 
The reader may find it useful to refer back to the notations from the proof of \thmref{thm:generalizaton} in Section~\ref{app:generalization}.  We shall show that when the dataset $\cs$ satisfies the ``nice'' properties outlined in Properties~\ref{p:noise.norm} through~\ref{p:sample.in.every.cluster}, then every KKT point of Problem~(\ref{eq:optimization problem}) is non-robust in the sense stated in the theorem.  By Proposition~\ref{prop:trainingdata}, the dataset $\cs$ satisfies these ``nice'' properties with probability 
at least $1- \left(3 n^2 d^{-\frac{\ln(d)}{2}} + n^2 e^{-d/16} + d^{1-\ln(d)} \right)$.


In the following lemma, we state several ``nice'' properties that are satisfied w.h.p. in a test example.

\begin{lemma} \label{lem:prob x is good}
	Suppose $\cs$ satisfies Properties~\ref{p:noise.norm} through~\ref{p:sample.in.every.cluster}.  With probability at least $1 - 4 n d^{-\frac{\ln(d)}{2}}$ over $\bx \sim \cd_\bx$ there exists $r \in Q$ such that the following hold:
	$\bx = \bmu^{(r)} + \bxi$ where for all $i \in I$ we have $| \inner{\bx_i,\bxi}| \leq \Delta$ and $| \inner{\bxi_i,\bxi}| \leq \Delta$. Also, $| \inner{\bx_i,\bx} - d | \leq \Delta$ for all $i \in I^{(r)}$ and $|\inner{\bx_i,\bx}| \leq p + \Delta$ for all $i \not \in I^{(r)}$. 
\end{lemma}
\begin{proof}
	The lemma follows immediately from \lemref{lem:random example is good}.	
\end{proof}

We next show that if the training data and test data are ``nice'', then KKT points have network outputs that are not too far from the margin. 

\begin{lemma} \label{lem:output not too large}
	Suppose $\cs$ satisfies Properties~\ref{p:noise.norm} through~\ref{p:sample.in.every.cluster}.  
 Let $\bx \in \reals^d$ and $r \in Q$ such that for all $i \in I^{(r)}$ we have $\inner{\bx,\bx_i} \in \left[d - \Delta, d + \Delta \right]$, and for all $q \in Q \setminus \{r\}$ and $i \in I^{(q)}$ we have $| \inner{\bx,\bx_i} | \leq p + \Delta$. Then, $| \cn_\btheta(\bx) | \leq 2$.
\end{lemma}
\begin{proof}
	We prove the claim for $r \in Q_+$. The proof for $r \in Q_-$ is similar. 
	By \lemref{lem:good example labeld correctly} we have $\cn_\btheta(\bx) > 0$.
	We now show that $\cn_\btheta(\bx) \leq 2$.
	
	By \eqref{eq:kkt condition w} and~(\ref{eq:kkt condition b}), for every $j \in J$ we have 
	\begin{align} \label{eq:single neuron separating r 3}
		\bw_j^\top \bx + b_j
		&= \left( \sum_{i \in I} \lambda_i y_i v_j \phi'_{i,j} \bx_i^\top \bx \right) + \sum_{i \in I} \lambda_i y_i v_j \phi'_{i,j} \nonumber
		\\
		&= \sum_{i \in I} \lambda_i y_i v_j \phi'_{i,j} (\bx_i^\top \bx + 1) \nonumber
		\\
		&= \left( \sum_{i \in I^{(r)}} \lambda_i v_j \phi'_{i,j} (\bx_i^\top \bx + 1)  \right) + \sum_{q \in Q \setminus \{r\}} \sum_{i \in I^{(q)}} \lambda_i y_i v_j \phi'_{i,j} (\bx_i^\top \bx + 1)~.
	\end{align}
	
	Now,
	\begin{align*}
		\cn_\btheta(\bx)
		&= \sum_{j \in J} v_j \phi(\bw_j^\top \bx + b_j)
		\leq \sum_{j \in J_+} v_j \left| \bw_j^\top \bx + b_j \right|~.
	\end{align*}
	By \eqref{eq:single neuron separating r 3} the above is at most 
	\begin{align*}
		\sum_{j \in J_+} &v_j \left[ \left( \sum_{i \in I^{(r)}} \lambda_i v_j \phi'_{i,j} (d + \Delta + 1)  \right) + \sum_{q \in Q \setminus \{r\}} \sum_{i \in I^{(q)}} \lambda_i v_j \phi'_{i,j} (p + \Delta + 1) \right]
		\\
		&= \sum_{j \in J_+} \left[ \left( \sum_{i \in I^{(r)}} \lambda_i v_j^2 \phi'_{i,j} (d + \Delta + 1)  \right) + \sum_{q \in Q \setminus \{r\}} \sum_{i \in I^{(q)}} \lambda_i v_j^2  \phi'_{i,j} (p + \Delta + 1) \right]
		\\
		&= \left[ (d + \Delta + 1) \sum_{i \in I^{(r)}}  \sum_{j \in J_+} \lambda_i v_j^2 \phi'_{i,j} \right] + \left[  (p + \Delta + 1) \sum_{q \in Q \setminus \{r\}} \sum_{i \in I^{(q)}} \sum_{j \in J_+}  \lambda_i v_j^2  \phi'_{i,j} \right]
		\\
		&\leq \left[ (d + \Delta + 1) \cdot \frac{1}{(1-2c')(d-\Delta + 1)}  \right] + \left[  (p + \Delta + 1) k \cdot \frac{1}{(1-2c')(d-\Delta + 1)} \right]~,
	\end{align*}
	where in the last inequality we used \lemref{lem:lam upper bound}. Plugging in $k = c' \cdot \frac{d - \Delta + 1}{p+\Delta+1}$, the above equals
	\begin{align*}
		&\left[\frac{d + \Delta + 1}{(1-2c')(d-\Delta + 1)}  \right] + \left[  c' \cdot \frac{d - \Delta + 1}{p+\Delta+1} \cdot \frac{p + \Delta + 1}{(1-2c')(d-\Delta + 1)} \right]
		\\
		&= \frac{d + \Delta + 1}{(1-2c')(d-\Delta + 1)}  +  \frac{c'}{(1-2c')}~.
	\end{align*}
	Using $c' \leq \frac{1}{10}$, the above is at most 
	\begin{align*}
		\frac{5(d + \Delta + 1)}{4(d-\Delta + 1)}  +  \frac{1}{8} 
		\leq \frac{5(d + \Delta)}{4(d-\Delta)}  +  \frac{1}{8}~.
	\end{align*}
    By \lemref{lem:bound Delta},
    the displayed equation is at most
	\[
		\frac{5(22d/21)}{4(20d/21)} + \frac{1}{8} 
		= \frac{5 \cdot 22}{4 \cdot 20}  +  \frac{1}{8} 
		= \frac{3}{2}
		\leq 2~. 
	\]
\end{proof}

Next, we show that the inputs to the neurons are not too negative for nice test examples.  

\begin{lemma} \label{lem:neuron input not small}
	Suppose $\cs$ satisfies Properties~\ref{p:noise.norm} through~\ref{p:sample.in.every.cluster}.  
 Let $r \in Q$ and let $\bx = \bmu^{(r)} + \bxi$ such that for all $i \in I$ we have $| \inner{\bx_i,\bxi}| \leq \Delta$ and $| \inner{\bxi_i,\bxi}| \leq \Delta$. Also, assume that $\inner{\bx_i,\bx} \in \left[d - \Delta, d + \Delta \right]$ for all $i \in I^{(r)}$ and $|\inner{\bx_i,\bx}| \leq p + \Delta$ for all $i \not \in  I^{(r)}$.
	Then, for all $j \in J$ we have
	\[
		\bw_j^\top \bx + b_j \geq - \sum_{i \in I} \lambda_i | v_j | \phi'_{i,j} (2 \Delta + p + 1)~.
	\]
\end{lemma}
\begin{proof}
	Suppose towards contradiction that there exists $j \in J$ such that
	\begin{equation} \label{eq: small inputs to neurons assumption}
		\bw_j^\top \bx + b_j 
		<  - \sum_{i \in I} \lambda_i | v_j | \phi'_{i,j} (2 \Delta + p + 1)
		\leq - \sum_{i \in I} 2 \lambda_i | v_j | \phi'_{i,j} \Delta~.
	\end{equation}
	Suppose first that $y^{(r)} v_j < 0$. By  \eqref{eq:kkt condition w} for all $i' \in I^{(r)}$ we have
	\begin{align} \label{eq:neuron input not small derivation 1}
		\bw_j^\top& \bx_{i'} + b_j 
		= \bw_j^\top \bx + b_j + \bw_j^\top (\bx_{i'} - \bx) \nonumber
		\\
		&= \bw_j^\top \bx + b_j + \bw_j^\top (\bmu^{(r)} + \bxi_{i'} - \bmu^{(r)} - \bxi) \nonumber
		\\
		&<  - \sum_{i \in I} 2 \lambda_i | v_j | \phi'_{i,j} \Delta + \sum_{i \in I} \lambda_i y_i v_j \phi'_{i,j} \bx_i^\top (\bxi_{i'} - \bxi) \nonumber
		\\
		&=  - \sum_{i \in I} 2 \lambda_i | v_j | \phi'_{i,j} \Delta -  \left( \sum_{i \in I} \lambda_i y_i v_j \phi'_{i,j} \bx_i^\top  \bxi \right) + \left( \sum_{i \in I \setminus \{i'\}} \lambda_i y_i v_j \phi'_{i,j} \bx_i^\top \bxi_{i'}  \right) +  \lambda_{i'} y_{i' }v_j \phi'_{i',j} \bx_{i'}^\top \bxi_{i'} \nonumber
		\\
		&\leq  - \sum_{i \in I} 2 \lambda_i | v_j | \phi'_{i,j} \Delta + \left( \sum_{i \in I} \lambda_i | v_j | \phi'_{i,j} \Delta \right) + \left( \sum_{i \in I \setminus \{i'\}} \lambda_i y_i v_j \phi'_{i,j} \bx_i^\top \bxi_{i'}  \right)  +  \lambda_{i'} y^{(r)} v_j \phi'_{i',j} \bx_{i'}^\top \bxi_{i'}~.
	\end{align}
	Recall that by our assumption on $\cs$, for every $i \in I$ and $q \in Q$ we have $| \inner{\bmu^{(q)},\bxi_i} | \leq \sigma \sqrt{d} \ln(d)$, and for every $s \neq s'$ in $I$ we have $| \inner{\bxi_s, \bxi_{s'} } | \leq \sigma^2 \sqrt{2d} \ln(d)$. 
	Thus, for $q \in Q$ and $i \in I^{(q)}$ such that $i \neq i'$ we have 
	\[
		| \bx_i^\top \bxi_{i'} |
		\leq | (\bmu^{(q)})^\top \bxi_{i'} | + | \bxi_i^\top \bxi_{i'} |
		\leq \sigma \sqrt{d} \ln(d) + \sigma^2 \sqrt{2d} \ln(d)
		\leq \Delta~.
	\]
	Moreover,
	\[
		\bx_{i'}^\top \bxi_{i'} 
		= (\bmu^{(r)})^\top \bxi_{i'} + \bxi_{i'}^\top \bxi_{i'} 
		\geq (\bmu^{(r)})^\top \bxi_{i'}
		\geq - \sigma \sqrt{d} \ln(d)
		\geq - \Delta~.
	\]
	Using the above displayed equations and the assumption $y^{(r)} v_j < 0$, the RHS in \eqref{eq:neuron input not small derivation 1} is at most	
	\begin{align*}
		- \sum_{i \in I} &2 \lambda_i | v_j | \phi'_{i,j} \Delta +  \left( \sum_{i \in I} \lambda_i | v_j | \phi'_{i,j} \Delta \right) + \left( \sum_{i \in I \setminus \{i'\}} \lambda_i | v_j | \phi'_{i,j} \Delta \right) + \lambda_{i'} | v_j | \phi'_{i',j} \Delta 
		\\
		&= - \sum_{i \in I} 2 \lambda_i | v_j | \phi'_{i,j} \Delta +  2 \left( \sum_{i \in I} \lambda_i | v_j | \phi'_{i,j} \Delta \right) 
		\\
		&= 0~.
	\end{align*}
	Hence, $y^{(r)} v_j < 0$ implies that $\phi'_{i',j} = \onefunc[\bw_j^\top \bx_{i'} + b_j \geq 0] = 0$ for all $i' \in I^{(r)}$.
	
	By \eqref{eq:kkt condition w} and~(\ref{eq:kkt condition b}) we have
	\begin{align} \label{eq:neuron input main}
		\bw_j^\top \bx + b_j 
		&=  \sum_{i \in I} \lambda_i y_i v_j \phi'_{i,j} \bx_i^\top \bx + \sum_{i \in I} \lambda_i y_i v_j \phi'_{i,j} \nonumber
		\\
		&=  \sum_{i \in I} \lambda_i y_i v_j \phi'_{i,j} (\bx_i^\top \bx + 1) \nonumber
		\\
		&= \left[ \sum_{i \in I^{(r)}} \lambda_i y^{(r)} v_j \phi'_{i,j} (\bx_i^\top \bx + 1) \right] + \left[ \sum_{q \in Q \setminus \{r\}}  \sum_{i \in I^{(q)}} \lambda_i y_i v_j \phi'_{i,j} (\bx_i^\top \bx + 1)  \right]~.
	\end{align}
	We now analyze both terms in the above RHS.
	
	Note that if $y^{(r)} v_j \geq 0$ then 
	\[
		 \sum_{i \in I^{(r)}} \lambda_i y^{(r)} v_j \phi'_{i,j} (\bx_i^\top \bx + 1) 
		 \geq \sum_{i \in I^{(r)}} \lambda_i y^{(r)} v_j \phi'_{i,j} (d - \Delta + 1) 
		 \geq 0~,
	\]
	and if $y^{(r)} v_j < 0$ then  we have 
	\[
		 \sum_{i \in I^{(r)}} \lambda_i y^{(r)} v_j \phi'_{i,j} (\bx_i^\top \bx + 1) 
		 = 0
	\]
	since $\phi'_{i',j} = 0$ for all $i' \in I^{(r)}$.
	Moreover, 
	\begin{align*}
		 \sum_{q \in Q \setminus \{r\}}  \sum_{i \in I^{(q)}} \lambda_i y_i v_j \phi'_{i,j} (\bx_i^\top \bx + 1)
		 &\geq - \sum_{q \in Q \setminus \{r\}}  \sum_{i \in I^{(q)}} \lambda_i | v_j | \phi'_{i,j} (p + \Delta + 1)
		 \\
		 &\geq - \sum_{i \in I} \lambda_i | v_j | \phi'_{i,j} (p + \Delta + 1)~.
	\end{align*} 
	Plugging the above equations into \eqref{eq:neuron input main} we get 
	\[
		\bw_j^\top \bx + b_j 
		\geq  - \sum_{i \in I}  \lambda_i | v_j | \phi'_{i,j} (p + \Delta + 1)
		\geq - \sum_{i \in I}  \lambda_i | v_j | \phi'_{i,j} (p + 2\Delta + 1)~,
	\]
	in contradiction to \eqref{eq: small inputs to neurons assumption}. 
\end{proof}

We now obtain a lower bound for the rate that perturbations in the direction $\bu = \sum_{r \in Q} y^{(r)} \bmu^{(r)}$ change the inputs to the neurons. 

\begin{lemma} \label{lem:neuron input changes fast}
	Suppose $\cs$ satisfies Properties~\ref{p:noise.norm} through~\ref{p:sample.in.every.cluster}.  
 Let $\bu = \sum_{r \in Q} y^{(r)} \bmu^{(r)}$. For every $j \in J_+$ we have 
	\[
		\bw_j^\top \bu \geq \sum_{i \in I} \lambda_i v_j \phi'_{i,j}  \left( d-\Delta - kp - k \Delta \right)~.
	\] 
	For every $j \in J_-$ we have 
	\[
		\bw_j^\top \bu \leq \sum_{i \in I} \lambda_i v_j \phi'_{i,j}  \left( d-\Delta - kp - k \Delta \right)~.
	\]
\end{lemma}
\begin{proof}
	Let $j \in J$.
	Using \eqref{eq:kkt condition w} we have
	\begin{align} \label{eq:uz bound}
		\bw_j^\top &\sum_{r \in Q} y^{(r)} \bmu^{(r)}
		= \sum_{i \in I} \lambda_i y_i v_j \phi'_{i,j} \bx_i^\top  \sum_{r \in Q} y^{(r)} \bmu^{(r)} \nonumber
		\\
		&= \sum_{q \in Q} \sum_{i \in I^{(q)}} \lambda_i y^{(q)} v_j \phi'_{i,j}  \bx_i^\top  \left(  y^{(q)} \bmu^{(q)} + \sum_{r \in Q \setminus \{q\}} y^{(r)} \bmu^{(r)} \right) \nonumber
		\\
		&= \sum_{q \in Q} \sum_{i \in I^{(q)}} \lambda_i v_j \phi'_{i,j}  \left(  (y^{(q)})^2 \bx_i^\top \bmu^{(q)} + \sum_{r \in Q \setminus \{q\}} y^{(q)} y^{(r)} \bx_i^\top \bmu^{(r)} \right) \nonumber
		\\
		&= \sum_{q \in Q} \sum_{i \in I^{(q)}} \lambda_i v_j \phi'_{i,j}  \left[  (\bmu^{(q)})^\top \bmu^{(q)} + \bxi_i^\top \bmu^{(q)} + \sum_{r \in Q \setminus \{q\}} \left( y^{(q)} y^{(r)} (\bmu^{(q)})^\top \bmu^{(r)} + y^{(q)} y^{(r)} \bxi_i^\top \bmu^{(r)} \right) \right]~. 
	\end{align}
	
	For $j \in J_+$ the above is at least
	\begin{align*}
	    \sum_{q \in Q} \sum_{i \in I^{(q)}} \lambda_i v_j \phi'_{i,j}  \left( d - \Delta - \sum_{r \in Q \setminus \{q\}} (p + \Delta) \right)
 		\geq \sum_{i \in I} \lambda_i v_j \phi'_{i,j}  \left( d - \Delta - k(p + \Delta) \right)~.        
	\end{align*}
	
	Similarly, for $j \in J_-$, \eqref{eq:uz bound} is at most
	\begin{align*}
	    \sum_{q \in Q} \sum_{i \in I^{(q)}} \lambda_i v_j \phi'_{i,j}  \left( d - \Delta - \sum_{r \in Q \setminus \{q\}} (p + \Delta) \right)
 		\leq \sum_{i \in I} \lambda_i v_j \phi'_{i,j}  \left( d - \Delta - k(p + \Delta) \right)~.
	\end{align*}
	
\end{proof}

We next show that the quantity $d-\Delta-kp-k\Delta$ appearing in the lemma above is strictly positive. 
\begin{lemma} \label{lem:expression is positive}
	We have $d - \Delta - k(p + \Delta) > 0$.
\end{lemma}
\begin{proof}
	By \assref{ass:dist} and \lemref{lem:bound Delta}, 
	we have 
	\begin{align*}
		k(p + \Delta)
		&< k(p + \Delta + 1)
		\leq \frac{1}{10} \cdot (d - \Delta + 1)
		= (d - \Delta) - \left[ \frac{9}{10} (d - \Delta) - \frac{1}{10} \right]
		\\
		&\leq (d - \Delta) - \left[ \frac{9}{10} \cdot \frac{20d}{21} - \frac{1}{10} \right]
		\leq (d - \Delta) - \left[ \frac{18 \cdot 1}{21} - \frac{1}{10} \right]
		< d - \Delta~.
	\end{align*}
\end{proof}

Using the above lemmas, we show that a small perturbation to nice inputs suffices for obtaining positive inputs to all hidden neurons.
  
\begin{lemma} \label{lem:neurons input become positive}
	Suppose $\cs$ satisfies Properties~\ref{p:noise.norm} through~\ref{p:sample.in.every.cluster}.  
 Let $\bz = \eta \sum_{q \in Q} y^{(q)} \bmu^{(q)}$ for $\eta \geq  \frac{ 2 \Delta + p + 1}{d-\Delta - kp - k \Delta}$. 
	Let $r \in Q$ and let $\bx = \bmu^{(r)} + \bxi$ such that for all $i \in I$ we have $| \inner{\bx_i,\bxi}| \leq \Delta$ and $| \inner{\bxi_i,\bxi}| \leq \Delta$. Also, assume that $\inner{\bx_i,\bx} \in \left[d - \Delta, d + \Delta \right]$ for all $i \in I^{(r)}$ and $|\inner{\bx_i,\bx}| \leq p + \Delta$ for all $i \not \in  I^{(r)}$.
	Then, 
	we have for all $j \in J_-$ that $\bw_j^\top (\bx - \bz) + b_j \geq 0$, and for all $j \in J_+$ that $\bw_j^\top (\bx + \bz) + b_j \geq 0$.
\end{lemma}
\begin{proof}
	Let $j \in J_-$. By \lemref{lem:neuron input not small}, we have 
	\[
		\bw_j^\top \bx + b_j \geq 
		\sum_{i \in I} \lambda_i  v_j \phi'_{i,j} (2 \Delta + p + 1)~.
	\] 
	By \lemref{lem:neuron input changes fast}, we have
	\[
		- \bw_j^\top \bz 
		\geq -\eta \sum_{i \in I} \lambda_i v_j \phi'_{i,j}  \left( d-\Delta - kp - k \Delta \right)~.
	\]
	Combining the last two displayed equations, we get
	\begin{align*}
		\bw_j^\top (\bx - \bz) + b_j
		&= \bw_j^\top \bx + b_j - \bw_j^\top \bz
		\\
		&\geq \sum_{i \in I} \lambda_i  v_j \phi'_{i,j} (2 \Delta + p + 1) - \eta \sum_{i \in I} \lambda_i v_j \phi'_{i,j}  \left( d-\Delta - kp - k \Delta \right)
		\\
		&= \sum_{i \in I} \lambda_i  v_j \phi'_{i,j} \left( 2 \Delta + p + 1 - \eta (d-\Delta - kp - k \Delta) \right)~. 
	\end{align*}
	Note that by \lemref{lem:expression is positive} we have $d-\Delta - kp - k \Delta > 0$. Hence, for $\eta \geq \frac{ 2 \Delta + p + 1}{d-\Delta - kp - k \Delta}$ we have $\bw_j^\top (\bx - \bz) + b_j \geq 0$.
	
	The proof for $j \in J_+$ is similar. Namely, by Lemmas~\ref{lem:neuron input not small} and~\ref{lem:neuron input changes fast}, we have
	\begin{align*}
	    \bw_j^\top (\bx + \bz) + b_j 
	    &= \bw_j^\top \bx + b_j + \bw_j^\top \bz
		\\
		&\geq - \sum_{i \in I} \lambda_i  v_j \phi'_{i,j} (2 \Delta + p + 1) + \eta \sum_{i \in I} \lambda_i v_j \phi'_{i,j}  \left( d-\Delta - kp - k \Delta \right)
		\\
		&= \sum_{i \in I} \lambda_i  v_j \phi'_{i,j} \left( - 2 \Delta - p - 1 + \eta (d-\Delta - kp - k \Delta) \right)~, 
	\end{align*}
	and hence for $\eta \geq \frac{ 2 \Delta + p + 1}{d-\Delta - kp - k \Delta}$ we get $\bw_j^\top (\bx + \bz) + b_j \geq 0$.
\end{proof}

Let now $\eta_1 =  \frac{ 2 \Delta + p + 1}{d-\Delta - kp - k \Delta}$ and $\eta_2 =  \frac{3(3d+\Delta+1)(1-2c')}{(d-\Delta -kp -k\Delta)(1-3c')ck}$. Note that by \lemref{lem:expression is positive} both $\eta_1$ and $\eta_2$ are positive. We denote $\bz = (\eta_1 + \eta_2) \sum_{q \in Q} y^{(q)} \bmu^{(q)}$.  

In the next lemma, we show that perturbing nice test examples from positive clusters with the vector $-\bz$ changes the sign of the network output. 

\begin{lemma} \label{lem:pert flips1}
	Suppose $\cs$ satisfies Properties~\ref{p:noise.norm} through~\ref{p:sample.in.every.cluster}.  
 Let $r \in Q_+$ and let $\bx = \bmu^{(r)} + \bxi$ such that for all $i \in I$ we have $| \inner{\bx_i,\bxi}| \leq \Delta$ and $| \inner{\bxi_i,\bxi}| \leq \Delta$. Also, assume that $\inner{\bx_i,\bx} \in \left[d - \Delta, d + \Delta \right]$ for all $i \in I^{(r)}$ and $|\inner{\bx_i,\bx}| \leq p + \Delta$ for all $i \not \in  I^{(r)}$.
	Then,  $\cn_\btheta(\bx - \bz) \leq -1$.
\end{lemma}
\begin{proof}
	We denote $\bx' = \bx - \bz$.
	By \lemref{lem:neuron input changes fast}, for every $j \in J_+$ we have 
	\begin{align*}
		\bw_j^\top  \bx' + b_j 
		&= \bw_j^\top \bx + b_j - \bw_j^\top   (\eta_1 + \eta_2) \sum_{q \in Q} y^{(q)} \bmu^{(q)}
		\\
		&\leq \bw_j^\top \bx + b_j  - (\eta_1 + \eta_2)   \sum_{i \in I} \lambda_i v_j \phi'_{i,j}  \left( d-\Delta - kp - k \Delta \right)~.
	\end{align*}
	By \lemref{lem:expression is positive} we get 
	\begin{equation} \label{eq:value does not increase1}
	    \bw_j^\top  \bx' + b_j 
	    \leq \bw_j^\top \bx + b_j. 
	\end{equation}
	Thus, in the neurons $J_+$ the input does not increase when moving from $\bx$ to $\bx'$.
	
	Consider now $j \in J_-$. Let $\tilde{\bx} = \bx - \eta_1  \sum_{q \in Q} y^{(q)} \bmu^{(q)}$. By \lemref{lem:neurons input become positive}, we have 
	$\bw_j^\top \tilde{\bx} + b_j \geq 0$. Also, by \lemref{lem:neuron input changes fast}, we have 
	\begin{align*}
		\bw_j^\top  \tilde{\bx} + b_j 
		&= \bw_j^\top \bx + b_j - \bw_j^\top \eta_1 \sum_{q \in Q} y^{(q)} \bmu^{(q)}
		\\
		&\geq \bw_j^\top \bx + b_j  - \eta_1 \sum_{i \in I} \lambda_i v_j \phi'_{i,j}  \left( d-\Delta - kp - k \Delta \right)~,
	\end{align*}
	and by \lemref{lem:expression is positive} the above is at least $\bw_j^\top \bx + b_j$. Thus, when moving from $\bx$ to $\tilde{\bx}$ the input to the neurons $J_-$ can only increase, and at $\tilde{\bx}$ it is 
	non-negative. 
	
	Next, we move from $\tilde{\bx}$ to $\bx'$. We have
	\begin{align*}
		\bw_j^\top \bx' + b_j
		= \bw_j^\top \tilde{\bx} + b_j - \eta_2 \bw_j^\top  \sum_{q \in Q} y^{(q)} \bmu^{(q)}
		\geq \max \left\{0,   \bw_j^\top \bx + b_j  \right\} - \eta_2 \bw_j^\top  \sum_{q \in Q} y^{(q)} \bmu^{(q)}~.
	\end{align*}
	By \lemref{lem:neuron input changes fast}, the above is at least 
	\begin{equation} \label{eq:value increases a lot1}
		 \max \left\{0, \bw_j^\top \bx + b_j  \right\} - \eta_2 \sum_{i \in I} \lambda_i v_j \phi'_{i,j}  \left( d-\Delta - kp - k \Delta \right) 
		 \geq 0~,
	\end{equation}
	where in the last inequality we use \lemref{lem:expression is positive}.
	
	Overall, we have 
	\begin{align*}
		\cn_\btheta(\bx') 
		&= \left[ \sum_{j \in J_+} v_j \phi(\bw_j^\top \bx' + b_j) \right] + \left[ \sum_{j \in J_-}  v_j \phi(\bw_j^\top \bx' + b_j) \right]
		\\
		&\overset{(i)} = \left[ \sum_{j \in J_+} v_j \phi(\bw_j^\top \bx' + b_j) \right] + \left[ \sum_{j \in J_-}  v_j (\bw_j^\top \bx' + b_j) \right]
		\\
		&\overset{(ii)}  \leq \left[ \sum_{j \in J_+} v_j \phi(\bw_j^\top \bx + b_j) \right] + \\
		&\;\;\;\;\; \left[ \sum_{j \in J_-}  v_j \left(   \max \left\{0,   \bw_j^\top \bx + b_j  \right\} - \eta_2 \sum_{i \in I} \lambda_i v_j \phi'_{i,j}  \left( d-\Delta - kp - k \Delta \right) \right) \right]~,
	\end{align*}
    where in $(i)$ we used \eqref{eq:value increases a lot1}, and in $(ii)$ we used both \eqref{eq:value does not increase1} and \eqref{eq:value increases a lot1}.
    Now, the above equals	
	\begin{align*}
		&\left[ \sum_{j \in J} v_j \phi(\bw_j^\top \bx + b_j) \right] - \left[ \sum_{j \in J_-}  v_j \eta_2  \sum_{i \in I} \lambda_i v_j \phi'_{i,j}  \left( d-\Delta - kp - k \Delta \right)  \right]
		\\
		&= \cn_\btheta(\bx) - \eta_2  \left( d-\Delta - kp - k \Delta \right)  \left[ \sum_{q' \in Q} \sum_{i \in I^{(q')}} \sum_{j \in J_-} \lambda_i v_j^2 \phi'_{i,j}  \right]~.
	\end{align*}
	
	Combining the above with \lemref{lem:output not too large}, \lemref{lem:lam lower bound}, and \lemref{lem:expression is positive}, we get 
	\begin{align*}
		\cn_\btheta(\bx') 
		&\leq 2 - \eta_2  \left( d-\Delta - kp - k \Delta \right)  \left[ \sum_{q' \in Q_-} \sum_{i \in I^{(q')}} \sum_{j \in J_-} \lambda_i v_j^2 \phi'_{i,j}  \right]
		\\
		&\leq 2 - \eta_2   \left( d-\Delta - kp - k \Delta \right) | Q_- |  \left( 1- \frac{c'}{1-2c'} \right) \frac{1}{3d+\Delta + 1}
		\\
		&\leq 2 - \eta_2   \left( d-\Delta - kp - k \Delta \right) ck  \left( \frac{1-3c'}{1-2c'} \right) \frac{1}{3d+\Delta + 1}~.
	\end{align*}
	For 
	\[
		\eta_2 
		= \frac{3(3d+\Delta+1)(1-2c')}{(d-\Delta -kp -k\Delta)(1-3c')ck} 
	\] 
	we conclude that $\cn_\btheta(\bx')$ is at most $-1$.
\end{proof}

Next, we show that perturbing nice test examples from negative clusters with the vector $\bz$ changes the sign of the network output. 
\begin{lemma} \label{lem:pert flips2}
Suppose $\cs$ satisfies Properties~\ref{p:noise.norm} through~\ref{p:sample.in.every.cluster}.  
	Let $r \in Q_-$ and let $\bx = \bmu^{(r)} + \bxi$ such that for all $i \in I$ we have $| \inner{\bx_i,\bxi}| \leq \Delta$ and $| \inner{\bxi_i,\bxi}| \leq \Delta$. Also, assume that $\inner{\bx_i,\bx} \in \left[d - \Delta, d + \Delta \right]$ for all $i \in I^{(r)}$ and $|\inner{\bx_i,\bx}| \leq p + \Delta$ for all $i \not \in  I^{(r)}$.
	Then,  $\cn_\btheta(\bx + \bz) \geq 1$.
\end{lemma}
\begin{proof}
    The proof follows similar arguments to the proof of \lemref{lem:pert flips1}. We provide it here for completeness.
    
    We denote $\bx' = \bx + \bz$.
	By \lemref{lem:neuron input changes fast}, for every $j \in J_-$ we have 
	\begin{align*}
		\bw_j^\top  \bx' + b_j 
		&= \bw_j^\top \bx + b_j + \bw_j^\top (\eta_1 + \eta_2) \sum_{q \in Q} y^{(q)} \bmu^{(q)}
		\\
		&\leq \bw_j^\top \bx + b_j + (\eta_1 + \eta_2) \sum_{i \in I} \lambda_i v_j \phi'_{i,j}  \left( d-\Delta - kp - k \Delta \right)~.
	\end{align*}
	By \lemref{lem:expression is positive} we get 
	\begin{equation} \label{eq:value does not increase2}
	    \bw_j^\top  \bx' + b_j 
	    \leq \bw_j^\top \bx + b_j. 
	\end{equation}
	Thus, in the neurons $J_-$ the input does not increase when moving from $\bx$ to $\bx'$.
	
	Consider now $j \in J_+$. Let $\tilde{\bx} = \bx + \eta_1  \sum_{q \in Q} y^{(q)} \bmu^{(q)}$. By \lemref{lem:neurons input become positive}, we have 
	$\bw_j^\top \tilde{\bx} + b_j \geq 0$. Also, by \lemref{lem:neuron input changes fast}, we have 
	\begin{align*}
		\bw_j^\top  \tilde{\bx} + b_j 
		&= \bw_j^\top \bx + b_j + \bw_j^\top \eta_1 \sum_{q \in Q} y^{(q)} \bmu^{(q)}
		\\
		&\geq \bw_j^\top \bx + b_j + \eta_1 \sum_{i \in I} \lambda_i v_j \phi'_{i,j}  \left( d-\Delta - kp - k \Delta \right)~,
	\end{align*}
	and by \lemref{lem:expression is positive} the above is at least $\bw_j^\top \bx + b_j$. Thus, when moving from $\bx$ to $\tilde{\bx}$ the input to the neurons $J_+$ can only increase, and at $\tilde{\bx}$ it is non-negative. 
	
	Next, we move from $\tilde{\bx}$ to $\bx'$. We have
	\begin{align*}
		\bw_j^\top \bx' + b_j
		= \bw_j^\top \tilde{\bx} + b_j + \eta_2 \bw_j^\top  \sum_{q \in Q} y^{(q)} \bmu^{(q)}
		\geq \max \left\{0,   \bw_j^\top \bx + b_j  \right\} + \eta_2 \bw_j^\top  \sum_{q \in Q} y^{(q)} \bmu^{(q)}~.
	\end{align*}
	By \lemref{lem:neuron input changes fast}, the above is at least 
	\begin{equation} \label{eq:value increases a lot2}
		 \max \left\{0, \bw_j^\top \bx + b_j  \right\} + \eta_2 \sum_{i \in I} \lambda_i v_j \phi'_{i,j}  \left( d-\Delta - kp - k \Delta \right) 
		 \geq 0~,
	\end{equation}
	where in the last inequality we use \lemref{lem:expression is positive}.
	
	Overall, we have 
	\begin{align*}
		\cn_\btheta(\bx') 
		&= \left[ \sum_{j \in J_-} v_j \phi(\bw_j^\top \bx' + b_j) \right] + \left[ \sum_{j \in J_+}  v_j \phi(\bw_j^\top \bx' + b_j) \right]
		\\
		&\overset{(i)} = \left[ \sum_{j \in J_-} v_j \phi(\bw_j^\top \bx' + b_j) \right] + \left[ \sum_{j \in J_+}  v_j (\bw_j^\top \bx' + b_j) \right]
		\\
		&\overset{(ii)} \geq \left[ \sum_{j \in J_-} v_j \phi(\bw_j^\top \bx + b_j) \right] + \\
		&\;\;\;\;\; \left[ \sum_{j \in J_+}  v_j \left( \max \left\{0,   \bw_j^\top \bx + b_j  \right\} + \eta_2 \sum_{i \in I} \lambda_i v_j \phi'_{i,j}  \left( d-\Delta - kp - k \Delta \right) \right) \right]~,
	\end{align*}
	where in $(i)$ we used \eqref{eq:value increases a lot2}, and in $(ii)$ we used both \eqref{eq:value does not increase2} and \eqref{eq:value increases a lot2}.
	Now, the above equals
	\begin{align*}
		&\left[ \sum_{j \in J} v_j \phi(\bw_j^\top \bx + b_j) \right] + \left[ \sum_{j \in J_+}  v_j \eta_2  \sum_{i \in I} \lambda_i v_j \phi'_{i,j}  \left( d-\Delta - kp - k \Delta \right)  \right]
		\\
		&= \cn_\btheta(\bx) + \eta_2  \left( d-\Delta - kp - k \Delta \right)  \left[ \sum_{q' \in Q} \sum_{i \in I^{(q')}} \sum_{j \in J_+} \lambda_i v_j^2 \phi'_{i,j}  \right]~.
	\end{align*}
	
	Combining the above with \lemref{lem:output not too large}, \lemref{lem:lam lower bound}, and \lemref{lem:expression is positive}, we get 
	\begin{align*}
		\cn_\btheta(\bx') 
		&\geq -2 + \eta_2 \left( d-\Delta - kp - k \Delta \right)  \left[ \sum_{q' \in Q_+} \sum_{i \in I^{(q')}} \sum_{j \in J_+} \lambda_i v_j^2 \phi'_{i,j}  \right]
		\\
		&\geq -2 + \eta_2 \left( d-\Delta - kp - k \Delta \right) | Q_+ |  \left( 1 - \frac{c'}{1-2c'} \right) \frac{1}{3d+\Delta + 1}
		\\
		&\geq -2 + \eta_2 \left( d-\Delta - kp - k \Delta \right) ck  \left( \frac{1-3c'}{1-2c'} \right) \frac{1}{3d+\Delta + 1}~.
	\end{align*}
	Plugging-in $\eta_2$,
	we conclude that $\cn_\btheta(\bx')$ is at least $1$.
\end{proof}

Finally, we show that the scale of the perturbation $\bz$ is small when $k$ is large. 

\begin{lemma} \label{lem:norm z}
	We have $\norm{\bz} = \co \left( \sqrt{\frac{d}{c^2 k}} \right)$. 
\end{lemma}
\begin{proof}
	We have
	\[
		\norm{\bz}^2
		= (\eta_1+\eta_2)^2 \norm{ \sum_{q \in Q} y^{(q)} \bmu^{(q)}}^2
	\]
	
	Now,
	\begin{align*}
		(\eta_1+ \eta_2)^2
		&= \left( \frac{ 2 \Delta + p + 1}{d-\Delta - kp - k \Delta} +  \frac{3(3d+\Delta+1)(1-2c')}{(d-\Delta -kp -k\Delta)(1-3c')ck} \right)^2
		\\
		&\leq \left( \frac{  2 (p + \Delta + 1)}{d-\Delta - k (p + \Delta + 1)} +  \frac{3(3d+\Delta+1)(1-2c')}{(d-\Delta - k (p + \Delta + 1))(1-3c')ck} \right)^2
		\\
		&\overset{(i)} = \left( \frac{2 c' (d-\Delta+1)/k }{d-\Delta - c' (d - \Delta + 1)} +  \frac{3(3d+\Delta+1)(1-2c')}{(d-\Delta -  c' (d - \Delta + 1))(1-3c')ck} \right)^2
		\\
		&\overset{(ii)} \leq \left(  \frac{1}{k} \cdot \frac{ \frac{2}{10} (d+1) }{d - \frac{d}{21} - \frac{1}{10} (d + 1)} + \frac{1}{ck} \cdot \frac{3(3d + \frac{d}{21} + 1)(1-\frac{2}{10})}{(d - \frac{d}{21} -  \frac{1}{10} (d + 1))(1-\frac{3}{10})} \right)^2
		\\
		&\leq \left( \frac{1}{k} \cdot \frac{\co(d) }{\Omega(d)} + \frac{1}{ck} \cdot \frac{\co(d)}{\Omega(d)} \right)^2
		\\
		&\leq \co\left(\frac{1}{c^2 k^2} \right)~,
	\end{align*}
	where in $(i)$ we used $k = c' \cdot \frac{d - \Delta + 1}{p+\Delta+1}$, and in $(ii)$ we used \lemref{lem:bound Delta}.
	
	Moreover,
	\begin{align*}
		\norm{ \sum_{q \in Q} y^{(q)} \bmu^{(q)}}^2
		&= \sum_{r \in Q} \sum_{q \in Q} y^{(r)} y^{(q)} \inner{\bmu^{(r)}, \bmu^{(q)}}
		\\
		&= \sum_{r \in Q} \left[\norm{\bmu^{(r)}}^2 + \sum_{q \neq r} y^{(r)} y^{(q)} \inner{\bmu^{(r)}, \bmu^{(q)}} \right]
		\\
		&\leq kd + k^2 p
		\\
		&= kd + k^2 \left( \frac{c'(d-\Delta+1)}{k} - \Delta - 1 \right)
		\\
		&\leq kd + kc'(d-\Delta+1) 
		\\ 
		&\leq \co(kd)~.
	\end{align*}
	
	Overall, we get
	\[
		\norm{\bz}^2 \leq \co \left( \frac{d}{c^2 k} \right)~.
	\]
\end{proof}

The theorem now follows immediately from Lemmas~\ref{lem:prob x is good},~\ref{lem:pert flips1},~\ref{lem:pert flips2}, and~\ref{lem:norm z}.

\end{document}